%% file: main.tex
\documentclass[11pt]{article}
\usepackage[letterpaper, left=1in, right=1in, top=1in,bottom=1in]{geometry}
\PassOptionsToPackage{linesnumbered,noend}{algorithm2e} 
\PassOptionsToPackage{pagebackref=true}{hyperref} 

\makeatletter
\renewcommand{\@biblabel}[1]{[#1]\hfill}
\makeatother

\usepackage[ruled,noend]{algorithm2e} 

\SetAlFnt{\small}
\SetAlCapFnt{\small}
\SetAlCapNameFnt{\small}
\SetAlCapHSkip{0pt}
\IncMargin{-\parindent}

\usepackage[utf8]{inputenc}

\usepackage{epsfig,amstext,xspace}
\usepackage[noend]{algpseudocode}

\SetCommentSty{mycommfont}

\usepackage{booktabs} 
\usepackage{wrapfig}

\usepackage[utf8]{inputenc}

\usepackage{geometry}
\usepackage[numbers]{natbib}

\usepackage{graphpap,amscd,mathrsfs,graphicx,lscape,enumitem,dsfont,bm,url,subcaption}
\usepackage{epsfig,amstext,xspace}
\usepackage{mdframed}
\usepackage{comment}
\usepackage{amsthm}

\usepackage{thmtools,thm-restate}

\usepackage[noend]{algpseudocode}
\usepackage{auxiliary}
\usepackage[utf8]{inputenc} 
\usepackage[T1]{fontenc}    
\usepackage[pagebackref=true]{hyperref}       
\usepackage{url}            
\usepackage{amsfonts}       
\usepackage{microtype}      
\usepackage{upgreek}

\makeatletter
\def\thm@space@setup{%
 \thm@preskip=\parskip \thm@postskip=0pt
}
\makeatother

\newcommand{\calA}{\mathcal{A}}
\newcommand{\calX}{\mathcal{X}}
\newcommand{\vx}{\mathbf{x}}

\newcommand{\1}{\mathds{1}}

\usepackage{color}              
\usepackage[suppress]{color-edits}
\definecolor{orange}{rgb}{0.9,0.5,0.0}
\addauthor{tl}{cyan}
\addauthor{ak}{orange}
\addauthor{cp}{black!20!red}

\begin{document}

\title{Contextual search in the presence of adversarial corruptions}


 \author{
  Akshay Krishnamurthy\thanks{Microsoft Research NYC, \texttt{akshaykr@microsoft.com }}
  \and 
 Thodoris Lykouris\thanks{Massachusetts Institute of Technology, \texttt{lykouris@mit.edu}. Research was initiated while the author was a postdoctoral research at Microsoft Research NYC.}  
 \and 
 Chara Podimata\thanks{Harvard University, \texttt{podimata@g.harvard.edu}. Research was initiated while the author was an intern at Microsoft Research NYC. The author is supported in part under grant No. CCF-1718549 of the National Science Foundation, and the Harvard Data Science Initiative.}
 \and
 Robert Schapire\thanks{Microsoft Research NYC, \texttt{schapire@microsoft.com}}
 }
 \date{First version: February 2020\\Current version: August 2022
\footnote{The first version was titled \emph{Corrupted multidimensional binary search: Learning in the presence of irrational agents}. An $8$-page extended abstract titled \emph{Contextual search in the presence of irrational agents} \citep{KrishnamurthyLykourisPodimataSchapireSTOC21} appeared at the 53rd ACM Symposium on the Theory of Computing (STOC '21).}}
\maketitle

\begin{abstract}
\input{abstract}
\end{abstract}
\addtocounter{page}{-1}
\thispagestyle{empty}

\newpage

\input{intro}
\input{prelim}
\input{algorithm_descr}
\input{analysis}
\input{gradient_descent}
\input{conclusion}

\bibliographystyle{alpha}
\bibliography{bibliog}

\newpage
\appendix
\input{app_glossary}

\input{app_related}

\input{app_projected_volume}
\input{app_single_dimension}
\input{app_analysis}
\input{app_bounded_rationality}
\input{app_discussion}
\input{app_GD}



\end{document}

%% file: abstract.tex
We study contextual search, a generalization of binary search in higher dimensions, which captures settings such as feature-based dynamic pricing. Standard formulations of this problem assume that agents act in accordance with a specific homogeneous response model. In practice however, some responses may be adversarially corrupted. Existing algorithms heavily depend on the assumed response model being (approximately) accurate for all agents and have poor performance in the presence of even a few such arbitrary misspecifications.

We initiate the study of contextual search when some of the agents can behave in ways inconsistent with the underlying response model. In particular, we provide two algorithms, one based on multidimensional binary search methods and one based on gradient descent. We show that these algorithms attain near-optimal regret in the absence of adversarial corruptions and their performance degrades gracefully with the number of such agents, providing the first results for contextual search in any adversarial noise model. Our techniques draw inspiration from learning theory, game theory, high-dimensional geometry, and convex analysis. 

%% file: intro.tex
\section{Introduction}
We study \emph{contextual search}, a fundamental problem that extends classical binary search to higher dimensions and has direct applications to pricing and personalized medicine~\citep{CLL16,BastaniBayati20,LLV18}.  In the most standard, linear version, at every round $t$, a \emph{context}  $\vx_t\in \R^d$ arrives.  Associated with this context is an unknown \emph{true value} $v_t\in\R$, which we here assume is a linear function of the context so that $v_t=\langle\stheta,\vx_t\rangle$ for some unknown vector $\stheta\in \R^d$, called the \emph{ground truth}. Based on the observed context $\vx_t$, the decision-maker or \emph{learner} selects a \emph{query} $\omega_t\in\R$ with the goal of minimizing some \emph{loss} that depends on the query as well as the true value; examples include the \emph{absolute loss}, $\abs{v_t-\omega_t}$, and the \emph{$\eps$-ball loss}, $\1\{\abs{v_t-\omega_t}>\eps\}$, both of which measure discrepency between $\omega_t$ and $v_t$.  Finally, the learner observes whether or not $v_t\geq\omega_t$, but importantly, the true value $v_t$ is never revealed, nor is the loss that was suffered.

For example, in feature-based dynamic pricing  \citep{CLL16,LLV18}, say, of Airbnb apartments, each context $\vx_t$ describes a particular apartment with components, or features, providing the apartment's location, cleanliness, and so on. The true value $v_t$ is the price an incoming customer or \emph{agent} is willing to pay, which is assumed to be a linear function (defined by $\stheta$) of $\vx_t$.  Based on $\vx_t$, the platform decides on a price $\omega_t$.  If this price is less than the customer's value $v_t$, then the customer makes a reservation, yielding revenue $\omega_t$; otherwise, the customer passes, generating no revenue.  The platform observes whether the reservation occurred (that is, if $v_t\geq \omega_t$).  The natural loss in this setting is called the \emph{pricing loss}, which captures how much revenue was lost relative to the maximum price that the customer was willing to pay.

A key challenge in contextual search is that the learner only observes \emph{binary feedback}, i.e., whether or not $v_t\geq \omega_t$. This contrasts with classical machine learning where the learner
observes either the entire loss function (full feedback) or only the loss itself for just the chosen query (bandit feedback). 

The above model makes the strong assumption that the feedback is always consistent with the ground truth. 
This is not always realistic as it does not account for model misspecifications. In particular, the agent's response may deviate \emph{arbitrarily} from the assumed linear model in some rounds. Such deviations can be modeled as adversarially corrupted feedback. 
Although 
prior works allow for stochastic noise (see Section~\ref{ssec:related_work}), they are 
not robust to any adversarial interference.

In this paper, we present the first contextual search algorithms that can handle adversarial noise. In particular, we allow some agents to behave in ways that are arbitrarily inconsistent with the ground truth. Inspired by the recent line of work on stochastic bandit learning with adversarial corruptions \citep{LML18, GKT19,ZimmertSeldin21}, we impose no assumptions on the order of corrupted rounds and obtain guarantees that gracefully degrade with their number while attaining near-optimality when all agents behave according to the linear model. 

\subsection{Our contributions}
We first provide a unifying framework encompassing disparate loss functions and agent response models (Section~\ref{sec:prelim}). In particular, we assume that the agent behaves according to a \emph{perceived value} $\tv$ and the precise response model determines the transformation from true to perceived value. The loss functions can depend on either the true value (to capture parameter estimation objectives) or the perceived value (for pricing objectives). This formulation allows us to capture adversarially corrupted agent responses, a setting not studied in prior work, as well as stochastic noise settings.

Our first algorithm (Section~\ref{sec:corpv_algorithm}) works for all of the aforementioned loss functions ($\eps$-ball, absolute, pricing). We prove that, with probability $1-\delta$, it
suffers  a relative degradation in performance or \emph{regret} of  $\bigO(C\cdot d^3\cdot \poly\log(T/\delta))$, where $C$ is the unknown number of adversarially corrupted rounds and $T$ is the total number of rounds (\emph{time horizon}). Our guarantee is logarithmic in $T$ when $C \approx 0$ and degrades gracefully as $C$ becomes larger. Our algorithm builds on the $\textsc{ProjectedVolume}$ algorithm \citep{LLV18}, which is optimal for the $\eps$-ball loss when $C=0$.

Our main technical advance is a method for maintaining a set of candidates for $\stheta$ (\emph{knowledge set}), successively removing candidates by hyperplane cuts while ensuring that $\stheta$ is never removed. When $C=0$, this is done via $\textsc{ProjectedVolume}$  which removes all parameters $\btheta$ that are inconsistent with the response in a way that each \emph{costly} query guarantees enough progress measured via the volume of the set of remaining parameters. However, when some responses are corrupted, such an aggressive elimination method may remove the ground truth $\stheta$ from the parameter space.

To deal with this key challenge, we run the algorithm in epochs, each corresponding to one query of  $\textsc{ProjectedVolume}$, and only proceed to the next epoch if we can find a hyperplane cut that makes volumetric progress and does not eliminate $\stheta$. We start from an easier setting where we assume a known upper bound $\bc$ on the number of corrupted responses ($C\leq \bc$) and only move to the next epoch when we can find a cut with enough volumetric progress that includes all parameters that are \emph{``misclassified''} by at most $\bc$ queries, i.e., parameters that were inconsistent with the agents' responses at most $\bc$ times. Note that $\stheta$ is consistent with all non-corrupted responses, and hence, it is always included in the new knowledge set (as it can only suffer at most $\bc$ misclassifications).

Our first challenge lies in identifying such a hyperplane cut, i.e., one that makes volumetric progress without eliminating $\stheta$. As discussed above, cuts associated with $\renato$ queries do make enough volumetric progress, but risk removing $\stheta$ due to their aggressive elimination. Interestingly, we can use ideas from convex analysis (specifically, the Carath\'eodory theorem) to show that, after collecting $\bigO(d^2 \bc )$ queries, we can combine them appropriately to produce the hyperplane cut with the desired properties. To guarantee the existence of such a cut, we identify a point in the parameter space that is outside of a convex body including all the \emph{protected parameters} (the ones with misclassification at most $\bc$) and then apply the separating hyperplane theorem.

A second challenge is that the separating hyperplane theorem does not provide a way to compute the corresponding cut. To deal with this, we use geometric techniques (volume cap arguments) to provide a sampling process that, with significant probability, identifies a point $\vq$ that is sufficiently far from the aforementioned separating hyperplane. We compute this hyperplane by running the classical learning-theoretic Perceptron algorithm repeatedly using points sampled from this process.

There are two remaining, intertwined challenges. On the one hand, the running time of Perceptron depends on the number of subregions created by removing all possible combinations of $\bc$ queries which is exponential in $\bc$. On the other hand, our algorithm needs to be agnostic to $\bc$. We deal with both of these via a multi-layering approach introduced in \citep{LML18} that runs multiple parallel versions of the aforementioned algorithm with only $\bc\approx \log T$ (Section~\ref{ssec:behavioral_results}). This results in a final algorithm that is quasipolynomial in the time horizon and does not assume knowledge of $C$.

Our second algorithm is based on gradient descent (Section~\ref{sec:gradient_descent}) and has a guarantee of $\bigO(\sqrt{T}+C)$ for the absolute loss. This algorithm is simple and efficient running time but does not provide logarithmic guarantees when $C\approx 0$ and does not extend to non-Lipschitz loss functions such as the pricing loss. The key idea in its analysis lies in identifying a simple proxy loss function based on which we can run gradient descent and directly apply its corresponding regret guarantee.

\subsection{Related work}\label{ssec:related_work}
Our work is closely related to dynamic pricing when facing an agent with \emph{unknown} demand curve. In the non-contextual version of the problem, there is a single item with infinite supply that is sold: the learner at each round posts a price for the item, and the agent decides whether to buy it or not based on their valuation function. This problem was formalized in the seminal work of Kleinberg and Leighton \cite{KL03} who studied the settings where the valuation is fixed, i.i.d., and adversarial and provided optimal pricing-loss regret guarantees of $\Theta(\log \log T), \Theta(\sqrt{T})$, and $\Theta(T^{2/3})$ respectively. Our work studies a contextual extension that falls between the first and the third category, since all agents behave according to the same valuation except for $C$ of them. 

\xhdr{Contextual search and dynamic pricing.} At a high level, there are two methodological approaches to handle the contextual setting. The first approach is based on binary search techniques and extends the first category described above where there exists a fixed parameter that, combined with the context, determines the valuation of the agent. Such techniques are very efficient in that they normally result in logarithmic regret guarantees and can handle contexts that are arbitrary and even selected by an adaptive adversary. This family of binary search approaches was introduced by Cohen, Lobel, and Paes Leme \cite{CLL16} who provided a binary search algorithm based on the ellipsoid method with a regret $\calO(d^2 \log (d/\eps))$ for the $\eps$-ball loss and $\calO(d^2\log T)$ for the symmetric and pricing loss. Lobel, Paes Leme, and Vladu \cite{LLV18} improved these bounds by obtaining the optimal regret $\calO(d \log (d/\eps))$ for the $\eps$-ball loss and regret $\calO(d \log T)$ for the symmetric and the pricing loss. Paes Leme and Schneider \cite{LS18} obtained regret guarantees of $\calO(d^4)$ and $\calO(d^4 \log \log (dT))$ for the symmetric and pricing loss, {which are optimal with respect to $T$}. Finally, Liu, Paes Leme, and Schneider \cite{LPLS20} obtained the optimal bounds {(with respect to both $d$ and $T$)} of $\calO(d \log d)$ and $\calO(d \log \log T)$ for the symmetric and the pricing loss respectively. These binary search techniques work by recursively refining a version space that contains the underlying parameter; this is what allows them to provide the logarithmic regret guarantees as they make exponential progress in refining the \emph{volume} of the version space at each round. This strength comes at a cost though in that it makes them very brittle even in the presence of a few corruptions. In particular, a single mistake may render the version space incorrect and the binary nature of the feedback makes it challenging to recover. Our work addresses this shortcoming by allowing some rounds to be arbitrarily (and even adversarially) corrupted and providing guarantees that gracefully degrade with the number of these corrupted rounds. Finally, most of the above works are not designed to handle even non-adversarial noise. The two exceptions are the work of Cohen, Lobel, and Paes Leme \cite{CLL16} who can handle a low-noise regime for all loss functions (in Appendix~
\ref{sec:extensions_behavior}, we extend our results to this setting) and the concurrent and independent work of Liu, Paes Leme, and Schneider \cite{LPLS20} whose results extend to a stochastic noise model where the feedback is flipped with a fixed, constant probability (their results only hold for the absolute loss and cannot handle adversarial noise).

The second methodological approach for contextual pricing is based on statistical methods such as linear regression and the central limit theorem \citep{GoldenshlugerZeevi13, BastaniBayati20, JN19, QB16, NSLW19,BK17, SBJ19, StatLearningPersonal}. These algorithms require the context to be i.i.d. and not adversarially selected as they separate exploration and exploitation in ways that are agnostic to the context. On the positive side, they are more robust to stochastic noise in the valuations and target the second category of valuations we discussed before, i.e., valuations that are i.i.d. Beyond the contextual setting, after the work of Kleinberg and Leighton \cite{KL03}, many papers incorporated important facets of dynamic pricing such as inventory constraints, multiple products, as well as different feedback and valuation models; see Appendix~\ref{app:related} for further discussion. 

\xhdr{Adversarial corruptions in learning with bandit feedback.} To capture adversarially corrupted agent responses, we posit that agents behave according to a fixed ground truth in all but $C$ rounds, during which they can deviate from it in arbitrary ways. The corruption budget $C$ can be selected adaptively and is unknown to the algorithm designer. This model was introduced by Lykouris, Mirrokni, and Paes Leme \cite{LML18} in the context of multi-armed bandits and their results for this setting were later strengthened by Gupta, Koren, and Talwar \cite{GKT19} and Zimmert and Seldin \cite{ZimmertSeldin21}. This model has been subsequently used for several other settings including linear optimization \citep{LLS19}, Gaussian bandit optimization \citep{bogunovic2020corruption}, assortment optimization \citep{CKW19}, reinforcement learning \citep{LykourisSSS2019CorruptionRL}, prediction with expert advice \citep{amir2020prediction}, learning product rankings \citep{GolrezaeiManSchSek21}, and dueling bandits \citep{AgarwalAgarwalPatil21}. Our work differs from these in that it involves a continuous action space which requires new analytical tools, while all prior results involve discrete (potentially large) action spaces. We note that a subsequent work by Chen and Wang \cite{Chen2020RobustDP} considers adversarial corruptions in a dynamic pricing setting with continuous actions. Our paper has orthogonal strengths: we consider a more complicated contextual setting and attain logarithmic regret, whereas they focus on incorporating inventory constraints.

Apart from corruptions, there are other multi-armed bandit approaches to go beyond i.i.d. rewards \citep{BrownianAlex, BubeckSli12, Gur1, Gur2, KeskinZeevi17, Cheungetal, COLT21Best}; see Appendix~\ref{app:related} for further discussion. 

\xhdr{Ulam's game and noisy binary search.} The non-contextual version of our problem bears similarities to \emph{Ulam's game} \citep{UlamOG}, where one wants to make the least number of queries to an adversary in order to identify a target number from set $\{1,2,\ldots,n\}$. The adversary can only give binary feedback and may lie at most $C$ times over the course of the game, where $C$ is known to the learner (\cite{Spencer92}, see also \cite{pelc2002searching} for a comprehensive survey). Rivest, Meyer, Kleitman, Winklmann, and Spenser \cite{RMKWS90} provide the optimal query complexity for this problem which is $\Theta(\log n+C\log\log n+C\log C)$; further discussion on how this bound relates to our guarantees is provided at the end of Section~\ref{ssec:analysis_adv_irrationality}. The algorithm proposed is intuitively a halving-type algorithm keeping track of all possible, feasible configurations for the timing of the lies. Our work extends this seminal paper in three directions. First, we cover the case of \emph{unknown} $C$. Second, we look at the contextual version of the problem. Third, our main algorithm obtains no-regret guarantees not only for the symmetric and the absolute loss, but for the \emph{pricing} loss as well. To achieve these, our algorithms and proof techniques are completely different from \cite{RMKWS90}. Ulam's game has been studied in multiple different variants~\citep{Pelc87, SpencerWinkler, AslamDhagat91, Dagan2, KK07, Nowak08, Nowak09}; see Appendix~\ref{app:related} for further discussion.

\xhdr{Beyond our assumptions.} Our model relies heavily on two assumptions. First, we assume that agents' valuations at each non-corrupted round are linearly dependent on the observed context $\vx_t$ (possibly with the addition of a small i.i.d. idiosyncratic noise). To the best of our knowledge, this is the viewpoint taken by almost all prior work that considers binary feedback with the exception of the works of Mao, Paes Leme, and Schneider \cite{MLS18} who consider \emph{Lipschitz} dependence of the valuation in $\vx_t$ and Shah, Blanchet, and Johari \cite{SBJ19} who posit an exponential relationship to the context. The second main assumption of our work is that the agents are \emph{myopic}, i.e., they make decisions optimizing their utilities only for the current round, without caring about future rounds. While this is a common assumption in prior work, there have also been works on dynamic pricing mechanisms where the agents are assumed to be long-living/non-myopic and thus \emph{``strategic''} \citep{Aminetal13, ARS14, MohriMunoz14, MohriMunoz15, FeldmanKorenLivniMansourZohar16, Drutsa17, LiuetalNeurIPS18,  Golrezaei1, Golrezaei2, PKDD20, ZD20, AAAI21, LearningFromBids}. Most of these works assume that agents optimize an infinite-horizon, discounted utility when making decisions. This can be viewed as a structured version of corruption as the agents only lie if they benefit from that and, at a high level, the goal of the learner in these cases is to design algorithms that will induce (approximately) truthful behavior from the agents and thus remove their incentive to deviate from the behavioral model. In contrast, we allow for arbitrary misspecifications in particular rounds so our algorithms cannot completely eliminate deviations from the prescribed behavioral model but need to be able to handle such misspecifications.

%% file: prelim.tex
\section{Model}\label{sec:prelim}
In this section, we provide a general framework (Section~\ref{ssec:protocol}) that allows us to study contextual search under different agent response models (Section~\ref{subsec:behaviors}) and different loss functions (Section~\ref{subsec:loss}). To facilitate the reader, we include a glossary with all recurring notation in Appendix~\ref{app_glossary}.

\subsection{Protocol}\label{ssec:protocol}
We consider the following repeated interaction between the learner and nature. Following classical works in contextual search \citep{CLL16,LLV18} we assume that the learner has access to a parameter space $\calK= \{\vu \in \bbR^d: \|\vu\|_2 \leq 1\}$ and a context space $\calX = \{ \vx \in \bbR^d: \|\vx\|_2 = 1\}$. We denote by $\Omega = [0,1]$ the decision space of the learner and by $\calV = [0,1]$ {a value space}; in the pricing setting, $\Omega$ can be thought of as the set of possible prices available to the learner and $\calV$ as a set of values associated with incoming agents. Domain $\calV$ helps express both the true value of the agents and the perceived value driving their decisions. Finally, we consider an agent response model determining the transformation from the agent's \emph{true value} to a \emph{perceived value} that drives the decision at each round. All of the above are known to the learner throughout the learning process.

The setting proceeds for $T$ rounds. Before the first round, nature chooses a ground truth $\stheta\in \calK$; this is fixed across rounds and is \emph{not} known to the learner. This ground truth determines both the agent's \emph{true} value function $v:\calX\to\calV$ and the learner's loss function $\ell: \Omega \times \calV \times \calV\to [0,1]$. We note that both value and loss functions are also functions of the ground truth $\stheta$; given that $\stheta$ is fixed throughout this process, we drop the dependence on $\stheta$ to ease notation. The functional form of both $v(\cdot)$ and $\ell(\cdot)$ as a function of the ground truth $\stheta$ is known to the learner but the learner does not know $\stheta$. In what follows, we use $\textrm{sgn}$ to denote the sign function, i.e., $\textrm{sgn}(x) = 1$ if $x \geq 0$ and $-1$ otherwise. For each round $t = 1, \dots, T$: 
\begin{enumerate}
\item\label{step:context} Nature chooses (potentially adaptively and adversarially) and reveals context $\vx_t \in \calX$. 
\item\label{step:perceived} Nature chooses but \emph{does not} reveal a perceived value $\tv_t\in\calV$ based on the response model.
\item\label{step:feedback} Learner selects query point $\omega_t\in \Omega$ (in a randomized manner) and observes $y_t=\textrm{sgn} \left( \tv_t - \omega_t \right)$.
\item\label{step:loss} Learner incurs (but does \emph{not} observe) loss: $\ell(\omega_t, v(\vx_t), \tv_t) \in [0,1]$.  
\end{enumerate}
Nature is an adaptive adversary (subject to the agent response model), i.e., it knows the learner's algorithm along with the realization of all randomness up to and including round $t-1$ (i.e, it knows all $\omega_{\tau},\forall \tau\leq t-1$), but does not know the learner's randomness at the current round $t$. Moreover, the learner only observes the context $\vx_t$ and the \emph{binary} variable $y_t$ as described in Steps~\ref{step:context} and \ref{step:feedback} of the protocol, and has access to neither the perceived value $\tv_t$ nor the loss $\ell(\omega_t,v(\vx_t), \tv_t)$. Finally, in the pricing setting, $y_t$ corresponds to whether the agent of round $t$ made a purchase or not.


\subsection{Agent response 
models}\label{subsec:behaviors}
We assume that the agents' true value function is: $v(\vx) = \langle \vx, \stheta \rangle$ {for any $\vx\in\calX$} (i.e., independent of their response model). The agent response model affects the perceived value $\tv$ at round $t$, which then affects both the loss incurred and the feedback observed by the learner. The agent response model that is mostly studied in contextual search works is \emph{full rationality}. This assumes that the agent always behaves according to their true value, i.e., $\tv_t=v(\vx_t)=\langle\vx_t,\stheta\rangle$. In learning-theoretic terms, this consistency with respect to a ground truth is typically referred to as \emph{realizability}. 

Our main focus in this work is the study of \emph{adversarially corrupted agents}. There, nature selects the rounds where these agents arrive ($c_t = 1$ if adversarially corrupted agents arrive, else $c_t = 0$), together with an upper bound $C$ on this number of rounds (i.e., $\sum_{t \in [T]} c_t \leq C$). Neither the sequence $\{c_t\}_{t \in [T]}$ nor the number $C$ are ever revealed to the learner. If $c_t = 0$, then nature is constrained to $\tv_t=v(\vx_t)$, but can select adaptively and adversarially $\tv_t$ if $c_t=1$. This model is inspired by the model of adversarial corruptions in stochastic bandit learning \citep{LML18}.

Our results extend to \emph{bounded rationality} which posits that the perceived value is the true value plus some noise parameter. The noise parameter is drawn from a $\sigma$-subgaussian distribution $\subG(\sigma)$, \emph{fixed} across rounds and \emph{known} to the learner, i.e., nature selects it before the first round and reveals it. At every round $t$ a realized noise $\xi_t\sim \subG(\sigma)$ is drawn, but $\xi_t$ is never revealed to the learner. The agent's perceived value is then $\tv_t=v(\vx_t)+\xi_t$. This stochastic noise model has been studied in contextual search as a way to incorporate idiosyncratic market shocks \citep{CLL16}.

We note that, to ease presentation, our model treats agents as different but homogeneous: each of them interacts with the learner exactly once. The exact same model can also be used to model a single agent that shows up for all $T$ rounds but is myopic in his/her choices.

\subsection{Loss functions and objective}\label{subsec:loss}

We study three variants for the learner's loss function: the $\eps$-ball, the absolute, and the pricing loss. Abstracting away from $t$ subscripts and dependencies on contexts $\vx$, the loss $\ell(\omega,v,\tv)$ evaluates the loss of a query $\omega$ when the true value is $v$ and the perceived value is $\tv$.

The first class of loss functions includes parameter estimation objectives that estimate the value of $\stheta$. One such function is the \emph{$\eps$-ball loss} which is defined with respect to an accuracy parameter $\eps>0$. The $\eps$-ball loss is $1$ if the difference between the query point $\omega$ and the true value $v$ is larger than $\eps$ and $0$ otherwise. Formally, $\ell(\omega, v, \tv) = \1 \left\{ \left|v- \omega \right| \geq \eps \right\}$. Another parameter estimation loss function is the \emph{absolute} or \emph{symmetric} loss that captures the absolute difference between the query point and the true value, i.e., $\ell(\omega, v, \tv) = \left|v - \omega \right|$. The aforementioned loss functions are unobservable to the learner as the true value $v$ is latent; this demonstrates that binary feedback does not offer strictly more information than the bandit feedback as the latter reveals the loss of the selected query.

Another important objective in pricing is the revenue collected which is the price $\omega$ in the event that the purchase occurred, i.e., $\tv\geq \omega$. This can be expressed based on observable information by setting a reward equal to $\omega$ when $\tv\geq \omega$ and $0$ otherwise. However, having this as a comparator leads to a benchmark with high objective, which tends to hinder logarithmic performance guarantees that are typical in binary search and are enabled by the fact that the loss of the comparator is $0$. A loss function exploiting this structure is the \emph{pricing loss} which is defined as the difference between the highest revenue that the learner could have achieved at this round (the agent's \emph{perceived} value $\tv$) and the revenue that the learner currently receives, i.e., $\omega$ if a purchase happens, and $0$ otherwise. The outcome of whether a purchase happens or not is tied to whether $\omega$ is higher or smaller than the perceived value $\tv$. Putting everything together: $\ell(\omega, v, \tv) = \tv - \omega\cdot \1\left\{\omega \leq \tv\right\}.$ 

We remark that the $\eps$-ball and the absolute loss depend only on the true value $v$ (and not the perceived value $\tv$); indeed, when these losses are considered $\tv$ affects only the feedback that the learner receives. That said, we define $\ell(\cdot, \cdot, \cdot)$ with three arguments for unification purposes, since the pricing loss does depend on the feedback that the learner receives (and hence, on $\tv$).

The learner's goal is to minimize a notion of regret. For adversarially corrupted
agents, the loss of the best-fixed policy in hindsight is \emph{at least} $0$ and \emph{at most} $C$. Hence, to simplify exposition, we slightly abuse notation and conflate loss and regret: $R(T) = \sum_{t \in [T]} \ell(\omega_t, v(\vx_t), \tv_t)$

It is no longer possible to provide sublinear guarantees for this quantity when facing boundedly rational agents and we therefore need to slightly relax the benchmark. To ease the exposition, we defer further discussion on the extension to bounded rationality to Appendix~\ref{sec:extensions_behavior}.

%% file: algorithm_descr.tex
\section{Corrupted Projected Volume: algorithm and main guarantee}\label{sec:corpv_algorithm}

In this section, we provide an algorithmic scheme that handles all the aforementioned agent response models and loss functions. The main result of this and the next section is an algorithm (Algorithm~\ref{algo:corpvAI}) for the adversarial corrupted agent response model when there is an \emph{unknown} upper bound $C$ on the number of corrupted agents. The regret of this algorithm is upper bounded by the following theorem. We first present the algorithm in this section and prove the stated theorem in Section~\ref{ssec:analysis_adv_irrationality}.

\begin{theorem}\label{thm:agnostic-corpv}
When run with an accuracy parameter $\eps > 0$ and an unknown corruption level $C$, $\textsc{CorPV.AC}$ incurs regret $\calO ( d^3 \cdot \log (\nicefrac{T}{\beta} ) \cdot \log (\nicefrac{d}{\eps} ) \cdot \log ( \nicefrac{1}{\beta}) \cdot (\log T + C ))$ with probability at least $1- \beta$ for the $\eps$-ball loss. When run with $\eps=1/T$, its regret for the pricing and absolute loss is $\calO (d^3 \log(dT) \log(T) \cdot (\log T + C) \log (\nicefrac{1}{\beta}))$ with probability at least $1 - \beta$. The expected runtime of the algorithm is quasi-polynomial; in particular, it is $\calO((d^2 \log T )^{\text{poly}\log T} \cdot \text{poly}(d,\log T) )$.
\end{theorem}

\subsection{Algorithm for the known-corruption setting}\label{ssec:algorithmic_crux}

A useful intermediate setting is the case where we know an upper bound $\bc$ on the number of adversarial agents, i.e. $C\leq \bc$; we refer to this as the $\bc$-\emph{known-corruption} setting. Our algorithm for this setting, \textsc{CorPV.Known} (Algorithm~\ref{alg:corpv_known}), builds on the $\textsc{ProjectedVolume}$
algorithm of Lobel, Paes Leme, and Vladu \cite{LLV18} which is optimal in terms of regret for the $\varepsilon$-ball loss when $\bc=0$.  
The main idea in \textsc{ProjectedVolume} is to maintain a \emph{knowledge set} $\calK_t$ which includes all candidate parameters $\btheta$ that are \emph{consistent} with what has been observed so far; all other parameters are eliminated from $\calK_t$. To be more concrete, a parameter $\btheta$ is consistent with observation $y_t$ with respect to a query point $\omega_t$, if $\sgn( \langle \vx_t , \btheta \rangle - \omega_t) = y_t$. The true parameter $\stheta$ is always consistent (since without corruptions $\tv_t=\langle \vx_t,\stheta\rangle$ and $y_t = \sgn(\tv_t-\omega_t)$) and therefore is never eliminated from the knowledge set. Further the volume of the knowledge set is intuitively a measure of progress for the algorithm. We now briefly describe the main components of \textsc{ProjectedVolume} and refer the reader to Appendix~\ref{app:proj-vol} for an algorithmic sketch and more details.

Given a context $\vx_t$, there are two scenarios. Before we describe them, we define the \emph{width} of a body $\calK_0$ in direction $\vx$ as $w(\calK_0,\vx)=\sup_{\btheta,\btheta'\in\calK_0}\langle\btheta-\btheta',\vx\rangle$. A small width along a certain direction $\vx$, means that we have adequately learned said direction, i.e., we do not need to refine our estimate of $\stheta$ further in this direction. Hence, if the width of the knowledge set in the direction of $\vx_t$ is $w(\calK_t,\vx_t)\leq \eps$, the algorithm makes an \emph{exploit} query $\omega_t=\langle\vx_t,\ttheta\rangle$ for \emph{any} point $\ttheta\in\calK_t$ which guarantees an $\eps$-ball loss equal to $0$. Otherwise, if $w(\calK_t,\vx_t)>\eps$, then the algorithm further refines the estimate of $\stheta$ in the direction of $\vx_t$ by making an \emph{explore} query $\omega_t=\langle \vx_t,\vkappa_t\rangle$, where $\vkappa_t$ is the (approximate) centroid of $\calK_t$. For a convex body $\calK_0$, the centroid is defined as $\vkappa^{\star}= \frac{1}{\texttt{\upshape vol} (\calK_0)} \int_{\calK_0} \vu d\vu$, where $\texttt{\upshape vol}(\cdot)$ denotes the volume of a set. Although computing the exact centroid of a convex set is $\#$P-hard \citep{Rademacher07}, one can efficiently approximate it~\citep{LLV18}. 

By querying $\omega_t=\langle \vx_t,\vkappa_t\rangle$, the algorithm learns that $\stheta$ lies in one of the two halfspaces passing through $\vkappa_t$ with normal vector $\vx_t$ , i.e., either $\langle \vx_t, \stheta \rangle \geq \omega_t$ or $\langle \vx_t, \stheta \rangle < \omega_t$. Then, it updates the knowledge set by taking intersection with this halfspace, i.e., all the parameters not in the intersection get eliminated from $\calK_{t+1}$. This ensures that the updated knowledge set still contains $\stheta$. We use $(\vh,\omega)$, $\vH^+\left(\vh,\omega \right)$, and $\vH^-\left(\vh, \omega\right)$ to denote the hyperplane with normal vector $\vh\in\bbR^d$ and intercept $\omega$, and the positive and negative halfspaces it creates with intercept $\omega$, i.e., $\vH^+(\vh, \omega) = \{\vx \in \bbR^d: \langle \vh, \vx \rangle \geq \omega\}$ and $\vH^-(\vh, \omega) = \{\vx \in \bbR^d: \langle \vh, \vx\rangle \leq \omega\}$. By properties of $\vkappa_t$, the volume of the {updated} knowledge set is a constant factor of the initial volume, leading to geometric volume progress. For technical reasons, \textsc{ProjectedVolume} keeps a set $S_t$ of dimensions with small width and works with the so-called \emph{cylindrification} $\Cyl(\calK_t,S_t)$ rather than the knowledge set $\calK_t$. Although we do the same to build on their analysis in a black-box manner, the distinction between $\calK_t$ and $\Cyl(\calK_t,S_t)$ is not important for understanding our algorithmic ideas and the corresponding definitions are deferred to Section~\ref{ssec:remaining_alg_components}.

Having described \textsc{ProjectedVolume} that works when there are no corruptions, we turn to our algorithm. In the presence of even a few corruptions, \textsc{ProjectedVolume} may quickly eliminate $\stheta$ from $\calK_t$ as in a corrupted round it may be that $\sgn(\langle \vx_t, \stheta \rangle-\omega_t) \neq y_t$ (see Appendix~\ref{app:single_dimension} for such an attack). To deal with this issue, we run the algorithm in epochs consisting of multiple queries. At each epoch, we combine all its queries to compute a hyperplane cut that both preserves $\stheta$ in the knowledge set and also makes enough volumentric progress on the latter's size. We face three important design decisions discussed separately below: what occurs inside an epoch, when to stop an epoch, and how to initialize the next one.

\xhdr{What occurs within an epoch?}
$\textsc{CorPV.Known}$ (Algorithm~\ref{alg:corpv_known}) formalizes what happens within an epoch $\phi$. The knowledge set is updated only at its end; this means that all rounds $t$ in epoch $\phi$ have the same knowledge set $\calK_\phi$ (and hence, the same centroid $\vkappa_\phi$). If the width of the knowledge set in the direction of $\vx_t$ is smaller than $\eps$, then, as in $\textsc{ProjectedVolume}$, we make an \emph{exploit}  query precisely described in Section~\ref{ssec:remaining_alg_components}. Otherwise, we make an \emph{explore} query $\omega_t=\langle \vx_t,\vkappa_\phi\rangle$, described below. The epoch keeps track of all explore queries that occur within its duration in a set $\mathcal{A}_\phi$. When it ends ($\phi'=\phi+1)$, the knowledge set $\calK_{\phi+1}$ of the new epoch is initialized. In this subsection, $\Cyl(\calK_\phi,S_\phi)$ can be thought as the knowledge set $\calK_\phi$ and the sets $S_\phi$ and $L_\phi$ can be ignored; {these quantities are needed for technical reasons and are discussed in Section~\ref{ssec:remaining_alg_components}.}
\begin{algorithm}[htbp]
\caption{$\textsc{CorruptedProjectedVolume-Known}$ ($\textsc{CorPV.Known}$)}
\label{alg:corpv_known}
\DontPrintSemicolon
\SetAlgoNoLine
\textbf{Global parameters:} Budget $\bc$, accuracy $\varepsilon$ \;
Initialize $\phi = 1, \calK_\phi \gets \calK,  S_\phi \gets \emptyset$, $\vkappa_{\phi} \gets \text{apx-centr}(\Cyl(\calK_{\phi}, S_\phi)), L_{\phi} \gets \text{orth-basis}(\bbR^d), \mathcal{A}_\phi \gets \emptyset$ \label{step:init_known}\;
\For{$t \in [T]$}{
    Observe context $\vx_t$ and set $\phi'\gets \phi$ \;
     \textbf{if }{$w(\Cyl(\calK_\phi,S_\phi),\vx_t)\leq \eps$ or $L_\phi= \emptyset$} \textbf{then}
    {query point $\omega_t=\textsc{CorPV.Exploit}(\vx_t,\calK_\phi)$\;
}    
    \textbf{else }{\label{step:geq-eps}
    $(\phi',\mathcal{A}_{\phi})\gets \textsc{CorPV.Explore}(\vx_t,\phi,\vkappa_\phi,L_{\phi},\mathcal{A}_{\phi})$\; }
    \If(\tcp*[f]{epoch changed in \textsc{CorPV.Explore}}){$\phi'=\phi+1$}{
        Compute separating cut: $(\tvh,\tomega)\gets \textsc{CorPV.SeparatingCut}(\vkappa_{\phi}, S_\phi, L_\phi, \mathcal{A}_\phi)$\;
        
        Make updates $(\calK_{\phi+1},S_{\phi+1},L_{\phi+1})\gets \textsc{CorPV.EpochUpdates}(\calK_{\phi}, S_\phi, L_\phi,\tvh,\tomega)$\;
        Initialize next epoch: $\phi\gets \phi'$, $\vkappa_{\phi} \gets \text{apx-centroid}(\Cyl(\calK_{\phi}, S_\phi))$, and $\mathcal{A}_{\phi}\gets \emptyset$.
        }
}
\end{algorithm}

$\textsc{CorPV.Explore}$ (Algorithm~\ref{alg:corpv-explore}) describes how we handle an
explore query. When $\bc = 0$, we can eliminate the halfspace that lies in the opposite direction of the feedback $y_t$ after each explore query. However, when $\bc>0$, this may eliminate $\stheta$. Instead, we keep all explore queries that occurred in  epoch $\phi$ as well as the halfspace consistent with the observed feedback in $\mathcal{A}_\phi$ and wait until we have enough data to identify a halfspace of the knowledge set that includes $\stheta$ and makes sufficient volumetric progress; we refer to this as a \emph{separating cut}. We then move to epoch $\phi'=\phi+1$. 

\begin{algorithm}[htbp]
\caption{$\textsc{CorPV.Explore}$}
\label{alg:corpv-explore}
\DontPrintSemicolon
\SetAlgoNoLine
\textbf{Parameters:} $\vx_t$, $\phi$, $\vkappa_\phi$, $L_{\phi}$, $\mathcal{A}_{\phi}$\;
        Select query point $\omega_t = \langle \vx_t, \vkappa_\phi \rangle$ and observe feedback $y_t$.\;
        Update set $\calA_\phi$: \textbf{if} $y_t = + 1$: $\mathcal{A}_\phi \gets \mathcal{A}_\phi \bigcup \vH^+ \left(\Pi_{L_\phi}\vx_t,\omega_t \right)$ \textbf{else} $\mathcal{A}_\phi \gets \mathcal{A}_\phi \bigcup \vH^- \left(\Pi_{L_\phi}\vx_t,\omega_t \right)$.\;
        \label{step:change_sign}
 \textbf{if} {$\abs{\mathcal{A}_\phi}\geq 
        \tau$} \textbf{then} {Move to next epoch $\phi'
        \gets\phi+1$.
      } \tcp*{$\tau:=2d\cdot \bc\cdot (d+1)+1$}
        \textbf{else }{Stay in the same epoch $\phi'\gets\phi$.}\;
        \Return $(\phi',\mathcal{A}_\phi)$ 
\end{algorithm}

\xhdr{When does the epoch end?} We next explain how many queries are enough to guarantee that such a separating cut exists. To guarantee that $\stheta$ is preserved after the cut, we need to make sure that we only eliminate the halfspace where candidate parameters $\btheta$ are misclassified (i.e., are inconsistent with the agent's response) by at least $\bc + 1$ explore queries. Note that because there are at most $\bc$ corruptions, $\stheta$ can be misclassified by at most $\bc$ queries. In other words, if the set of all candidate parameters $\btheta$ that are misclassified by at most $\bc$ explore queries are on the non-eliminated halfspace of the hyperplane, then this hyperplane can serve as a separating cut. We refer to the set of these parameters as the \emph{protected region} as we aim to ensure that they are not eliminated.

At first glance, one might think that, after $2\bc + 1$ explore queries, we can directly use one of them as a separating cut. Interestingly, although this is the case for $d=2$, we show that for $d=3$, if we are restricted to separating cuts on the direction of existing explore queries, even arbitrarily many such queries do not suffice (see Appendix~\ref{ssec:improper_cut}). One key technical component in our analysis is to show that when we combine $\tau = 2d (d+1) \bc + 1$ explore queries, there exists a hyperplane that separates the convex hull of the protected region from a point $\vpst$ that is close to $\vkappa_\phi$ but also outside of that convex hull (separating hyperplane theorem). Since that hyperplane crosses close to $\vkappa_\phi$, we ensure enough volumetric progress. Since the non-eliminated halfspace includes all the parameters in the protected region, we ensure that we do not eliminate $\stheta$. The proof of this argument relies on the Carath\'eodory theorem and is informally sketched in the left figure of Fig.~\ref{fig:calP}.

\xhdr{How do we initialize the next epoch?} Since the existence of the separating cut is established, if we were able to compute this cut, we would be able to compute the knowledge set of the next epoch by taking its intersection with the positive halfspace of the cut. However, the separating hyperplane theorem provides only an existential argument and no direct way to compute the separating cut. To deal with this, recall that the separating cut should have $\vpst$ on its negative halfspace and the whole protected region in its positive halfspace. To compute it, we use the Perceptron algorithm \citep{R58}, which is typically used to provide a linear classifier for a set of (positive and negative) points in the \emph{realizable} setting (i.e., when there exists a hyperplane that correctly classifies these points). Perceptron proceeds by iterating across the points and suggesting a classifier. Every time that a point is misclassified, Perceptron makes an update. If the entire protected region is classified as positive and $\vpst$ as negative by Perceptron, then we return its hyperplane as the separating cut; otherwise we feed one point that violates the intended labeling to Perceptron. Perceptron makes a mistake and updates its classifier. The main guarantee of Perceptron is that, if there exists a classifier with margin of $\gamma>0$ (i.e., smallest distance to any data point is $\gamma$), the number of mistakes that Perceptron makes is at most $1/\gamma^2$ (precisely, the bound is in Lemma~\ref{lem:perceptron}).

The problem is that we do not know $\vpst$ and, even if we deal with this, $\vpst$ does not necessarily have a large enough margin from the protected region. To overcome this, we provide a sampling process that with big enough probability identifies a different point $\vq$, \emph{in the vicinity} of $\vpst$, whose margin to the protected region is lower bounded by $\gamma$. If $\vq$ does have the desired margin, the mistake bound of Perceptron controls the running time needed to identify the separating hyperplane. Otherwise, we proceed with a new random point. This takes care of the small margin issue with $\vpst$.

In order to pin down $\vpst$, we construct
a set of points $\Land_\phi$, which we call \emph{landmarks}, such that at least one of them is outside of the convex hull of the protected region. We run multiple versions of Perceptron, each with a random $\vpst\in\Land_\phi$ and a point $\vq$ randomly selected in a ball around $\vpst$ of an appropriately defined radius $\zeta$, which we denote by $\calB_{L_\phi}(\vpst, \zeta)$; this can be computed efficiently by normalizing $\calB_{L_{\phi}}(\vpst,\zeta)$ to a unit ball and using the techniques presented in \cite[Section~2.5]{BHK16}. If $\vq$ has a big-enough margin then the mistake bound of Perceptron ensures that $\textsc{CorPV.SeparatingCut}$ (Algorithm~\ref{algo:separating_cut}) returns the separating cut. Volume cap arguments show that point $\vq$ has the required margin with big enough probability (informally sketched in the right figure of Fig.~\ref{fig:calP}), which bounds the number of the outer \emph{while} loops and thereby the running time. 

\begin{figure}
    \centering
    \includegraphics[width=\textwidth]{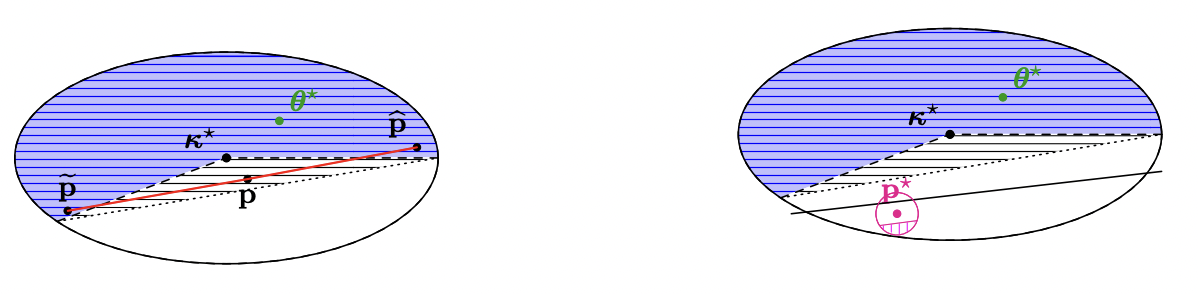}
    \caption{Informal sketch of the Carath{\'e}odory's theorem in two dimensions (left) and the computation of a separating cut (right). The ellipsoid corresponds to the current knowledge set $\calK$. The blue area corresponds to the $\bc$-protected region. The patterned area corresponds to its convex hull. In the left figure, any point $\vp$ of the convex hull can be written as a convex combination of points $\widehat{\vp}, \widetilde{\vp}$ from the protected region. Hence, all points inside the convex hull of the protected region (horizontal-line pattern) have been misclassified \emph{at most} $(d+1)\bc$ times. In the right figure, point $\vpst$ has been misclassified \emph{at least} $2d(d+1)\bc +1$. Sampling points from the spherical cap of the magenta ball (vertical-line pattern) gives a big enough margin for Perceptron. The black, solid line corresponds to the valid separating cut.}\label{fig:calP}
\end{figure}

\begin{algorithm}[htbp]
\caption{\textsc{CorPV.SeparatingCut}}\label{algo:separating_cut}
\DontPrintSemicolon
\SetAlgoNoLine
\textbf{Parameters:} $\vkappa_\phi,S_{\phi},L_\phi,\mathcal{A}_{\phi}$\tcp*[r]{size of small dimensions $\delta:=\frac{\eps}{4(d + \sqrt{d})}$}
\label{step:resample}
Fix landmarks $\Land_\phi = \left\{\vkappa_\phi \pm  \bnu \cdot e_i, \forall e_i \in E_\phi\right\}$ where $E_{\phi}\gets \text{orth-basis}(L_\phi)$ and ${\bnu} = \frac{\eps - 2 \sqrt{d} \cdot \delta}{4\sqrt{d}}$\label{step:landmarks}\;
Let $\vw := (\tvh, \tomega) \in \bbR^{d+1}$. \tcp*[r]{Perceptron hyperplane}
        \While{$\mathtt{true}$}{
   Initialize $\vw = [1]^d$ and mistake counter to $M \gets 0$. \tcp*[r]{Perceptron initialization}
    Sample a random point $\vq$ from ball $\calB_{L_{\phi}}(\vpst, \zeta)$ with radius $\zeta=\bnu$ around random $\vpst\in\Land_\phi$. \label{step:sample-ball}  \;
        \While(\tcp*[f]{Perceptron mistake bound}){$M < \frac{d-1}{\zeta^2\cdot \ln^2(3/2)}$\label{step:perc-mb}}{
        Set $m \gets 0$. \;
        \lIf(\tcp*[f]{Perceptron update}){$\vq \in \vH^{+}
        (\tvh,\tomega)$}{ $\vw \gets \vw - q$
        ; set $m \gets m +1$. \label{step:check-centr-pt1}
        }
        \lIf(\tcp*[f]{Perceptron update}){$\vkappa_\phi \in \vH^{-}
        (\tvh, \tomega)$} {$\vw \gets \vw + q$; set $m \gets m +1$. \label{step:check-centr-pt2}
        }
        \label{step:comb1}\For{subsets $D_\phi \subseteq \mathcal{A}_\phi$ such that $|D_\phi| = \bc$ \label{step:subsets}}{
            Let $P$ be the polytope created by halfspaces of $\mathcal{A}_\phi\setminus D_\phi$ and $\vH^-(\tvh,\tomega)$. \label{step:lp1} \;
            \lIf(\tcp*[f]{Perceptron update}){$P \neq \emptyset$}{
            $\vw \gets \vw + \vz$ for $\vz \in P$; 
            set $m \gets m+1$.\label{step:lp2}
            }
        }
        \lIf{$m \neq 0$}  {increase mistake counter $M \gets M+m$.} \label{step:check-calpc}
        \lElse {
        \textbf{return} $(\tvh, \tomega)$ 
        }
}}
\end{algorithm}

We discuss next the computational complexity of our algorithm. As written in lines \ref{step:subsets}-\ref{step:lp2} of Algorithm~\ref{algo:separating_cut}, checking whether the protected region is contained in the positive halfspace of the Perceptron hyperplane requires going over all $\binom{|\mathcal{A}_\phi|}{\bc}$ ways to remove $\bc$ hyperplanes and checking whether the resulting region intersects the negative halfspace (if this happens, then points with misclassification of at most $\bc$ may be misclassified). This suggests a running time that is exponential in $\bc$. Fortunately, as detailed in Section~\ref{ssec:behavioral_results}, to handle the unknown corruption or the other intricacies in our actual behavioral model beyond the $\bc$-known-corruption, we only run this algorithm with $\bc\approx\log(T)$. As a result, the final running time of our algorithms is quasi-polynomial in $T$.

\subsection{Adapting to an unknown corruption level}\label{ssec:behavioral_results}
We now provide the algorithm when the corruption level $C$ is unknown (Algorithm~\ref{algo:corpvAI}). The places where $\textsc{CorPV.AC}$ differs from $\textsc{CorPV.Known}$ are in lines \ref{step:layer_sampling}, \ref{step:first_robust_layer}-\ref{step:exploit_robust}, \ref{step:consistent_non_robust}-\ref{step:end_updates_non_robust}. This section extends ideas from \cite{LML18} for multi-armed bandits to contextual search which poses an additional difficulty as the search space is continuous. This is not as straightforward as a doubling trick for the unknown $C$, as both the loss and the corruption $c_t$ are unobservable; doubling tricks require identifying a proxy for the quantity under question and doubling once a threshold is reached.

The basic idea is to maintain multiple copies of $\textsc{CorPV.Known}$, which we refer to as \emph{layers}. At every round, we decide which copy to play probabilistically. Each copy $j$ keeps its own environment with its corresponding epoch $\phi(j)$ and knowledge set $\calK_{j,\phi(j)}$. Smaller values $j$ for the copies are \emph{less robust} to corruption and we impose a monotonicity property among them by ensuring that the knowledge sets are nested, i.e., $\calK_{j,\phi(j)}\subseteq \calK_{j',\phi(j')}$ for $j\leq j'$. This allows more robust layers to correct mistakes of less robust layers that may inadvertently eliminate $\stheta$ from their knowledge set.

More formally, we run $\log T$ parallel versions of the $\bc$-known-corruption algorithm with a corruption level of $\bc\approx\log(T)$. At the beginning of each round $t$, the algorithm randomly selects layer $j$ with probability $2^{-j}$ (line \ref{step:layer_sampling}) and executes the layer's algorithm for this round. Since the adversary does not know the randomness in the algorithm, this makes layers $j$ with $C\leq 2^j$ robust to corruption level of $C$. The reason is that the expected number of corruptions occurring at layer $j$ is at most $1$ and, with high probability, less than $\log T$ which is accounted by the $\bc=\log T$ upper bound on corruption based on which we run $\textsc{CorPV.Known}$ on this layer. 

However, there is a problem: all layers with $C>2^j$ are not robust to corruption of $C$ so they may eliminate $\stheta$ and, to make things worse, the algorithm follows the recommendation of these layers with large probability. As a result, we need a way to supervise their decisions by more robust layers. To achieve that, we use nested active sets; when the layer $j_t$ selected at round $t$ proceeds with a separating cut on its knowledge set, we also make the same cut on all less robust layers $j'<j_t$ (lines \ref{step:consistent_non_robust}-\ref{step:epoch_update_non_robust}). This allows non-robust layers that have eliminated $\stheta$ from their knowledge set to correct their mistakes by removing the incorrect parameters of their version space that they had converged to from their knowledge sets.
\begin{algorithm}[htbp]
\caption{$\textsc{CorPV.AC}$ (Adversarial Corruption version)}\label{algo:corpvAI}
\DontPrintSemicolon
\SetAlgoNoLine
\textbf{Global parameters:} Failure probability $\beta$, budget $\bc:=2\log (T/\beta)$, accuracy $\varepsilon$ \;
Initialize layer-specific quantities for all layers $j \in [\log T]$: $\phi(j) = 1$, $\calK_{j,\phi(j)} \gets \calK$, $S_{j,\phi(j)} \gets \emptyset,\vkappa_{j,\phi(j)} \gets \text{apx-centroid}(\Cyl(\calK_{j,\phi(j)}, S_{j,\phi(j)})), L_{j,\phi(j)} \gets \text{orthonorm-basis}(\bbR^d), \mathcal{A}_{j,\phi(j)} \gets \emptyset$\;
\For{$t \in [T]$}{
    Sample layer $j _t\in [\log T]$: $j_t=j$ with probability $2^{-j}$; with remaining probability, $j_t = 1$. \label{step:layer_sampling} \;
     Observe context $\vx_t$ and set $\phi'\gets \phi(j_t)$. \;
     \If{$w(\Cyl(\calK_{j_t,\phi(j_t)},S_{j_t,\phi(j_t)})\leq \eps$ or $L_{j_t,\phi(j_t)}\neq \emptyset$
     }
    {
        Find smallest more robust exploit layer $j \geq j_t$: $j=\min_{j'\geq j_t}{w(\calK_{j', \phi(j')},\vx_t) \leq \eps}$.\label{step:first_robust_layer}\;
    Compute \emph{exploit} query point for this layer: $\omega_t=\textsc{CorPV.Exploit}(\vx_t,\calK_{j,\phi(j)})$
    \label{step:exploit_robust}.
}    
    \textbf{else }$(\phi',\vpst,\mathcal{A}_{j_t,\phi(j_t)})\gets \textsc{CorPV.Explore}(\vx_t,\phi(j_t),\vkappa_{j_t,\phi(j_t)}, L_{j_t,\phi(j_t)},\mathcal{A}_{j_t,\phi(j_t)})$\; 
    \If(\tcp*[f]{epoch changed in \textsc{CorPV.Explore}}){$\phi'=\phi(j_t)+1$}{
      $(\tvh,\tomega)\gets \textsc{CorPV.SeparatingCut}(\vkappa_{j_t,\phi(j_t)}, S_{j_t,\phi(j_t)},L_{j_t,\phi(j_t)}, \mathcal{A}_{j_t,\phi(j_t)})$\;
     $(\calK_{j_t,\phi'},S_{j_t,\phi'},L_{j_t,\phi'})\gets \textsc{CorPV.EpochUpdates}(\calK_{j_t,\phi(j_t)}, S_{j_t,\phi(j_t)}, L_{j_t,\phi(j_t)},\tvh,\tomega)$\;
       $\phi(j_t)\gets \phi'$, $\vkappa_{j_t,\phi(j_t)} \gets \text{apx-centroid}(\Cyl(\calK_{j_t,\phi(j_t)}, S_{j_t,\phi(j_t)}))$, and  $\mathcal{A}_{j_t,\phi(j_t)}\gets \emptyset$.\;
    \For(\tcp*[f]{Make less robust layers consistent with $j_t$}){$j' \leq j_t$}{ \label{step:consistent_non_robust}
         $(\calK_{j',\phi(j')},S',L')\gets \textsc{CorPV.EpochUpdates}(\calK_{j',\phi(j')}, S_{j',\phi(j')}, L_{j',\phi(j')},\tvh,\tomega)$\label{step:epoch_update_non_robust}\;
        \If{$\vkappa_{j',\phi(j')}\notin \calK_{j',\phi(j')} \textbf{ or } S'\neq S_{j',\phi(j')}$ \label{step:check_for_update_epoch_non_robust}} {
       $\phi(j')\gets \phi(j')+1$, $(S_{j',\phi(j')+1},L_{j',\phi(j')+1})\gets (S',L')$, $\mathcal{A}_{j',\phi(j')}\gets \emptyset$\;
        $\vkappa_{j',\phi(j')} \gets \text{apx-centroid}(\Cyl(\calK_{j',\phi(j')}, S_{j',\phi(j')}))$. \label{step:end_updates_non_robust}}
      }
      }
}
\end{algorithm}
There are two additional points that arise in the contextual search setting. First, the aforementioned cut may not make enough volumetric progress in the knowledge sets of layers $j'<j_t$. As a result, as described in lines \ref{step:check_for_update_epoch_non_robust}-\ref{step:end_updates_non_robust}, we only move to the next epoch for layer $j'$ if its centroid is removed from the knowledge set or another change discussed in Section~\ref{ssec:remaining_alg_components} is triggered. Second, with respect to exploit queries, we want to make sure that we do not keep confidence on non-robust layers. As a result, we follow the exploit recommendation of the largest layer $j\geq j_t$ that has converged to exploit recommendation in this direction, i.e., $w(\calK_{j,\phi(j)})\leq \eps$ (lines \ref{step:first_robust_layer}- \ref{step:exploit_robust}). This eventually allows us to bound the regret from all non-robust layers by the smallest robust layer $\lceil\log C\rceil$ (see Section~\ref{ssec:analysis_adv_irrationality}).

\subsection{Remaining components of the algorithm.}\label{ssec:remaining_alg_components}
The presentation of the algorithm until this point has disregarded some technical parts. We now discuss each of them so that the algorithm is fully defined.

\xhdr{Cylindrification, small, and large dimensions.} To facilitate relating the volume progress to a bound on the explore queries, similar to \cite{LLV18}, we keep two sets of vectors/dimensions $S_\phi$ and $L_\phi$ whose union creates an orthonormal basis. The set $S_\phi$ has \emph{small dimensions} $\vs\in S_\phi$ with width $w(\calK_\phi,\vs)\leq\delta$ for $\delta:=\frac{\eps}{4(d+\sqrt{d})}$. The set $L_\phi$ is any basis for the subspace orthogonal to $S_\phi$, with the property that $\forall \vl \in L_\phi: w(\calK_\phi,\vl)> \delta$. The set $L_\phi$ completes an orthonormal basis maintaining that $\vl\in L_\phi: w(\calK_\phi,\vl)> \delta$. When an epoch ends, sets $S_\phi$ and $L_\phi$ are updated together with the knowledge set $\calK_\phi$ as described in $\textsc{CorPV.EpochUpdates}$ (Algorithm~\ref{alg:corpv_epoch_updates}): if the new direction $\tvh$ of the separating cut projected to the large dimensions has width $w(\Pi_{L_\phi}\calK_{\phi+1},\tvh)\leq \delta$, we add it to
$S_{\phi+1}$ and we update $L_{\phi+1}$ to keep the invariant that no large dimension has width larger than $\delta$.
\begin{algorithm}[htbp]
\caption{$\textsc{CorPV.EpochUpdates}$}
\label{alg:corpv_epoch_updates}
\DontPrintSemicolon
\SetAlgoNoLine
\textbf{Parameters: } $\calK_{\phi}, S_\phi, L_\phi,\tvh,\tomega$\;
Update $\calK_{\phi+1} \gets \calK_{\phi} \bigcap \vH^+(\tvh, \tomega)$ and save temporary sets $\tS \gets S_\phi$ and $\tL \gets L_\phi$. \;
        \label{step:small-large}\If(\tcp*[f]{size of small dimensions $\delta:=\frac{\eps}{4(d + \sqrt{d})}$}){$w\left(\Pi_{L_{\phi}}\calK_{\phi+1},\tvh\right) \leq \delta$}{
        Add hyperplane to small dimensions $\tS \gets S_{\phi} \bigcup \left\{\tvh\right\}$.\;
        Compute orthonormal basis for new large dimensions $\tL$ (without $\tS$).\;
        }
        Update $L_{\phi+1}\gets \tL \setminus \{e_i\in \tL: w(\calK_{\phi+1},e_i)\leq \delta\}$ and $S_{\phi+1}\gets \tS \bigcup \prn*{\tL \setminus L_{\phi+1}}$. \label{step:keep-large}\;
        \Return $(\calK_{\phi+1},S_{\phi+1},L_{\phi+1})$
\end{algorithm}

Overall, the potential function we use to make sure that we make progress depends on the projected volume of the knowledge set on the large dimensions $L_\phi$, as well as the number of small dimensions $S_\phi$. This is why in lines~\ref{step:first_robust_layer}- \ref{step:exploit_robust} of Algorithm~\ref{algo:corpvAI}, we update the epoch of less robust layers when one of these two measures of progress is triggered.
Sets $S_\phi$ and $L_\phi$ serve in explaining which dimensions are identified well enough so that we can focus our attention on making progress in the remaining dimensions. For this to happen, an important notion is that of \emph{Cylindrification} which creates a box covering the knowledge set and removes the significance of the small dimensions. 

\begin{definition}[Cylindrification, Definition~4.1 of \cite{LLV18}]
\label{def:cylindrification}
Given a set of orthonormal vectors $S = \{\vs_1, \dots, \vs_n\}$, let $L = \{ \vu | \langle \vu, \vs \rangle = 0; \forall \vs \in S\}$ be a subspace orthogonal to $\texttt{\upshape span}(S)$ and $\Pi_L \calK$ be the projection of convex set $\calK \subseteq \bbR^d$ onto $L$. We define:
\begin{equation*}
\Cyl(\calK,S) := \left\{\vz + \sum_{i = 1}^n b_i \vs_i \Big| \vz \in \Pi_L \calK \text{ and } \min_{\btheta \in \calK} \langle \btheta, \vs_i \rangle \leq b_i \leq \max_{\btheta \in \calK} \langle \btheta, \vs_i \rangle\right\}.
\end{equation*}
\end{definition}

By working with the Cylindrification $\Cyl(\calK_\phi,S_\phi)$ (Definition~\ref{def:cylindrification}) rather than the original set of small dimensions $S_\phi$, we can ensure that we make queries that make volumetric progress with respect to the large dimensions, that have been less well understood.
This is the reason why the landmark $\vpst$ we identify lives in the large dimensions while being close to the centroid $\vkappa_\phi$ (line~\ref{step:landmarks}).

\xhdr{Exploit queries for different loss functions.} When the width of the knowledge set on the direction of the incoming context is small, i.e., $w(\calK_\phi,\vx_t)\leq \eps$, we proceed with an exploit query. This module evaluates the loss of each query $\omega$ with respect to any parameter that is consistent with the knowledge set, i.e., $\stheta\in\calK_\phi$. It then employs a min-max approach by selecting the query $\omega_t$ that has the minimum loss for the worst-case selection of $\btheta\in\calK_\phi$. For the $\eps$-ball loss, any query point $\omega_t=\langle \vx_t, \btheta'\rangle$ with $\btheta'\in\calK_\phi$ results in loss equal to $0$; this is what $\textsc{ProjectedVolume}$ also does to achieve optimal regret for the $\eps$-ball loss function. 
\begin{algorithm}[htbp]
\caption{$\textsc{CorPV.Exploit}$}
\label{alg:corpv-exploit}
\DontPrintSemicolon
\SetAlgoNoLine
\textbf{Parameters:} $\vx_t$, 
$\calK_\phi$\;
        Compute  query point $\omega_t=\min_{\omega\in\Omega} \max_{\btheta\in\calK_\phi} \ell(\omega, \langle \btheta, \vx_t\rangle,\langle \btheta, \vx_t\rangle)$\;
        \Return $\omega_t$
\end{algorithm}

Moving to the pricing loss and assuming that the query point is $\omega_t = \langle \vx_t, \btheta \rangle$ for some $\btheta \in \calK_\phi$, although the distance of $\stheta$ to hyperplane $(\vx_t,\omega_t)$ is less than $\eps$, there is a big difference based on which side of the hyperplane $\stheta$ lies in (i.e., whether $\stheta \in \vH^+(\vx_t, \omega_t)$ or not). Specifically, if $\omega_t> \langle \vx_t, \stheta_t \rangle$ then a fully rational agent does not buy and we get zero revenue, thereby incurring a loss of $\langle \vx_t, \stheta \rangle$. On the other hand, querying $\omegast = \langle \vx_t, \stheta \rangle$ would lead to a purchase from a fully rational agent, and hence, to a pricing loss of $0$. As we discuss in Section~\ref{sec:gradient_descent}, this discontinuity in pricing loss poses further complications in extending other algorithms to contextual search. 

To deal with this discontinuity, we can query point $\omega_t$ with $\omega_t = \langle \vx_t, \stheta\rangle - \eps$, as the value of the fully rational agent is certainly above this price. In fact, when dealing with boundedly rational agents (Appendix~\ref{sec:extensions_behavior}), such a lower price is essential even if we know $\stheta$ in order to account for the noise and there the definition of $\omega_t$ accounts for the distributional information about the noise.

%% file: analysis.tex
\section{Analysis}\label{sec:corpv_analysis}
In this section we provide the analysis of the algorithm introduced in Section~\ref{sec:corpv_algorithm}. We first analyze the result for the intermediate
$\bc$-known corruption setting.  This setting allows us to introduce our key additional ideas and serves as a building block to extend to both the setting where $C$ is unknown (Theorem~\ref{thm:agnostic-corpv}) as well as the bounded rationality behavioral model (Theorem~\ref{thm:bounded}).

\subsection{Existence of a separating hyperplane at the end of any epoch.}\label{ssec:existence}
We first show, in Lemma~\ref{lem:existence}, that after $\tau = 2d \cdot \bc (d+1) + 1$ rounds, there exist $\vhst_\phi\in\bbR^d$ and $\omegast_\phi\in\bbR$ such that the hyperplane $\left(\vhst_\phi,\omegast_\phi\right)$ is a separating cut, i.e., it passes close to the approximate centroid $\vkappa_\phi$ (and therefore also to the centroid $\vkappast_\phi$), and has in the entirety of one of its halfspaces only parameters ``misclassified'' at least  $\bc+1$ explore times. The results of this subsection hold for \emph{any} scalar $\delta <\frac{\eps}{2\sqrt{d} + 4d}$. For the analysis, we make three simplifications (all without loss of generality) in an effort to ease the notation. First, we assume that $y_t = +1, \forall t \in [T]$. This is indeed without loss of generality since the algorithm can always negate the received {context} $\vx_t$ and the chosen {query} $\omega_t$ to force $y_t = +1$ (Step~\ref{step:change_sign} of Algorithm~\ref{alg:corpv-explore}). Second, for rounds where nature's answer is arbitrary, we assume that the perceived value is $\tv_t = \langle \vx_t, \ttheta \rangle$, where $\ttheta \in \calK_0$ and it can change from round to round. For all other rounds $\ttheta = \stheta$. Third, we assume that all hyperplanes have unit $\ell_2$ norm.  

\begin{lemma}\label{lem:existence}
For any epoch $\phi$, scalar $\delta \in \left(0,\frac{\eps}{2\sqrt{d} + 4d}\right)$, and scalar {$\bnu = \frac{\eps - 2 \sqrt{d} \cdot \delta}{4\sqrt{d}}$}, after $\tau=2 d \cdot \bc (d+1) + 1$ rounds, there exists a hyperplane $(\vhst_{\phi}, \omegast_\phi )$ \emph{orthogonal to all small dimensions $S_\phi$} 
such that the resulting halfspace $\vH^+(\vhst_\phi,\omegast_\phi)$ always contains $\stheta$ and $\dist(\vkappa_\phi, (\vhst_\phi, \omegast_\phi )) \leq {\bnu}$, where by $\dist(\vkappa, (\vh, \omega))$ we denote the distance of point $\vkappa$ from hyperplane $(\vh, \omega)$, i.e., $\dist(\vkappa, (\vh, \omega)) = \frac{|\langle \vkappa, \vh \rangle - \omega|}{\|\vh\|}$. 
\end{lemma}

\noindent At a high level, the tuning of $\bnu$ depends on two factors. First, in order to make sure that we make enough progress in terms of volume elimination, despite the fact that we do not make a cut through $\vkappa_\phi$, we need $\bnu$ to be close enough to $\vkappa_\phi$ (Lemma~\ref{lem:approx_grunbaum}). Second, we need to guarantee that there exists at least one point with a very high undesirability level (Lemma~\ref{lem:landmark2}). For the analysis, we define the $\nu$-margin projected undesirability levels, which we later use for some fixed $\nu<\bnu$:
\begin{definition}[$\nu$-Margin Projected Undesirability Level]\label{def:undesir-final}
Consider an epoch $\phi$, a scalar {$\nu$}, and a point $\vp$ in $\calK_\phi$. Given the set $\mathcal{A}_\phi = \{(\Pi_{L_\phi}\vx_t,\omega_t)\}_{t \in [\tau]}$, we define $\vp$'s \emph{$\nu$-margin projected undesirability level}, denoted by $u_\phi(\vp,\nu)$, as the number of rounds within epoch $\phi$, for which
\[u_\phi(\vp,\nu)=\sum_{t\in [\tau]}\1\left\{\left(\left\langle\vp-\vkappa_{\phi},\Pi_{L_{\phi}}\vx_t\right\rangle+\nu\right)<0\right\}.\]
\end{definition}
Intuitively, $u_\phi(\vp,\nu)$ gives penalty to a point $\vp$ if it is far (more than $\nu$) from the negative halfspace of the query (when projected to the large dimensions $L_\phi$). We can then show (Lemma~\ref{lem:undesir-gt}) that the undesirability level of a point $\vp$ during an epoch $\phi$ corresponds to the number of times during epoch $\phi$ that $\vp$ and $\ttheta$ were at opposite sides of hyperplane $(\Pi_{L_\phi}\vx_t, \nu + \omega_t )$ for any $\nu > \bnu$.

Armed with this, we define the $\bc$-protected region in large dimensions, $\calP(\bc, \nu)$, which is the set of points in $\calK_\phi$ with $\nu$-margin projected undesirability level at most~$\bc$. Mathematically: \[ \calP(\bc{,\nu})= \{\vp \in \calK_\phi: u_{\phi}(\vp, \nu) \leq \bc \}\]

The next lemma establishes that if we keep set $\calP(\bc,\nu)$ intact in the convex body formed for the next epoch $\calK_{\phi+1}$, then we are guaranteed to not eliminate point $\stheta$ (proof in Appendix~\ref{app:existence}).

\begin{lemma}\label{lem:protect-stheta}
If $\nu > \unu$ (where $\unu=\sqrt{d}\delta$), then the ground truth $\stheta$ is included in the set $\calP(\bc{,\nu})$.
\end{lemma}
We next show that there exists a hyperplane cut, that is orthogonal to all small dimensions in a way that guarantees that the set $\calP(\bc{,\nu})$ is \emph{preserved} in $\calK_{\phi+1}$ (i.e., $\calP(\bc,\nu) \subseteq \calK_{\phi+1}$). Note that due to Lemma~\ref{lem:protect-stheta}, it
is enough to guarantee that we have $\stheta \in \calK_{\phi+1}$. However, $\calP(\bc{,\nu})$ is generally non-convex and it is not easy to directly make claims about it. Instead, we focus on its convex hull, denoted by $\conv(\calP(\bc{,\nu}))$; for any point in $\conv(\calP(\bc,\nu))$ we can upper bound its undesirability by applying Carath\'eodory's Theorem, which says that any point in the convex hull of a (possibly non-convex) set can be written as a convex combination of at most $d+1$ points of that set. Using this result, we can bound the $\nu
$-margin projected undesirability levels of all the points in $\conv(\calP(\bc{,\nu}))$.

\begin{lemma}
For any scalar $\nu$, epoch $\phi$ and any point $\vp \in \conv(\calP(\bc{,\nu}))$, its $\nu$-margin projected undesirability level is \emph{at most} $\bc \cdot (d+1)$, i.e., $u_{\phi}(\vp, \nu) \leq \bc \cdot (d+1)$. 
\end{lemma}
\begin{proof}
From Carath{\'e}odory's Theorem, since $\vp \in \bbR^d$ and is inside $\conv(\calP(\bc{,\nu}))$, it can be written as the convex combination of \emph{at most} $d+1$ points in $\calP(\bc{,\nu})$. Denoting these points by $\{\vz_1, \dots, \vz_{d+1}\}$ such that $\vz_i \in \calP(\bc{,\nu}), \forall i \in [d+1]$, $\vp$ can be written as $\vp = \sum_{i = 1}^{d+1} a_i \vz_i$ where $a_i \geq 0, \forall i \in [d+1]$ and $\sum_{i=1}^{d+1} a_i = 1$. Hence, the $\nu$-margin projected undesirability level of $\vp$ in epoch $\phi$ is:  
\begin{align*}
u_\phi(\vp,\nu)&=\sum_{t\in [\tau]} \1\left\{ \left(\left\langle\vp-\vkappa_{\phi},\Pi_{L_{\phi}}\vx_t\right\rangle+\nu\right) <0\right\} &\tag{Definition~\ref{def:undesir-final}}\\
&= \sum_{t\in [\tau]} \1 \left\{\sum_{i \in [d+1]} a_i \underbrace{\left( \left\langle  \vz_i-\vkappa_\phi,\Pi_{L_\phi} \vx_t\right\rangle+\nu\right) }_{Q_i}<0 \right\} &\tag{Carath{\'e}odory's Theorem}\\ 
  & \leq \sum_{t \in [\tau]} \sum_{i \in [d+1]} \1 \left\{\left(\langle \vz_i - \vkappa_\phi,\Pi_{L_\phi}\vx_t\rangle + \nu \right) < 0\right\}\\
  & \leq \sum_{i\in[d+1]} u_\phi(\vz_i,\nu) \leq \bc \cdot (d+1) &\tag{$\vz_i \in \calP(\bc,\nu)$ and definition of $\calP(\bc,\nu)$}
\end{align*}
where the first inequality comes from the fact that if $Q_i \geq 0$ for \emph{all} $\vz_i, i \in [d+1]$, then the corresponding summand contributes $0$ undesirability points to $u_\phi(\vp,\nu)$, since $a_i \geq 0$ as this is a convex combination. As a result, each undesirability point on the left hand side of the latter inequality can be attributed to at least one $\vz_i$ from the right hand side.
\end{proof}

\noindent Next, we prove that there exists some point $\vq \in \calK_\phi$ such that $u_\phi(\vq,\nu) \geq \bc \cdot (d+1)+1$. Note that by the previous lemma, we know that $\vq \notin \conv(\calP(\bc,\nu))$.
As a result, \emph{any} hyperplane separating $\vq$ from $\conv(\calP(\bc{,\nu}))$ preserves $\calP(\bc{,\nu})$ (and as a result, $\stheta$) for $\calK_{\phi+1}$. To make sure that we also make progress in terms of volume elimination, we show below that there exists a separating hyperplane in the space of large dimensions (i.e., orthogonal to all small dimensions). For our analysis, we introduce the notion of \emph{landmarks}.
\begin{definition}[Landmarks]\label{def:landmark}
Let basis $E_{\phi} = \{\mathbf{e}_1, \dots, \mathbf{e}_{d - |S_\phi|}\}$ be such that $E_\phi$ is orthogonal to $S_\phi$, any scalar $\delta \in \left(0, \frac{\eps}{2\sqrt{d} + 4d}\right)$, and a scalar $\bnu = \frac{\eps - 2\sqrt{d} \delta}{4\sqrt{d}}$. We define the $2(d - |S_\phi|)$ \emph{landmarks} to be the points such that $\Lambda_\phi = \left\{\vkappa_\phi \pm  \bnu \cdot \mathbf{e}_i, \forall \mathbf{e}_i \in E_\phi\right\}$.
\end{definition}

Landmarks possess the convenient property that at every round where the observed context $\vx_t$ was such that $w(\calK_\phi,\vx_t) \geq \eps$, at least one of them gets a $\nu$-margin projected undesirability point, when $\nu < \bnu$ (Lemma~\ref{lem:landmark2}). The tuning of $\bnu$ explains the constraint imposed on $\delta$, i.e., $\delta < \frac{\eps}{2\sqrt{d} + 4\delta}$. This constraint is due to the fact that since $\nu > \unu$ and $\nu < \bnu$, then it must be the case that $\unu < \bnu$, where $\unu = \sqrt{d}\delta$ and $\bnu = \frac{\eps - 2\sqrt{d}}{4\sqrt{d}}$. Since, at every round at least one of the landmarks gets a $\nu$-margin projected undesirability point, then if we make $\tau$ sufficiently large, then, by the \emph{pigeonhole principle}, at least one of the landmarks has $\nu$-margin projected undesirability at least $\bc \cdot (d+1) +1$, which allows us to distinguish it from points in $\conv(\calP(\bc))$. Formally (with proof in \ref{app:existence}):

\begin{lemma}\label{lem:landmark3}
For scalar $\nu \in (\unu,\bnu)$, after $\tau=2d \cdot \bc \cdot (d+1) +1$ rounds in epoch $\phi$, there exists a landmark $\vp^{\star}\in\Land_\phi$ such that $\vp^{\star} \notin \conv(\calP(\bc {,\nu}))$.
\end{lemma}
We can now
prove the main lemma of this subsection. We note that during the computation of $\vhst_\phi$,  nature does not provide any new context $\vx$, and hence, we incur no additional regret.

\begin{proof}[Proof of Lemma~\ref{lem:existence}]
By Lemma~\ref{lem:landmark3}, {for $\bnu = \frac{\eps - 2 \sqrt{d} \cdot \delta}{4\sqrt{d}}$} and $\delta \in \left(0, \frac{\eps}{2\sqrt{d}+4d}\right)$, there exists a landmark $\vp^{\star}\in\Land_\phi$ that lies outside of $\conv(\calP(\bc{,\nu}))$. As a result, there exists a hyperplane separating $\vp^{\star}$ from the convex hull. We denote this hyperplane by $(\vhst_\phi,\omegast_\phi)$. Recall that since $\vp^{\star} \in \Land_\phi$ then by definition $\| \vkappa_\phi - \vp^{\star} \| = {\bnu}$. As the hyperplane separates $\vkappa_\phi$ from $\vp^\star$, it holds that
$\dist(\vkappa_\phi,(\vhst_\phi,\omegast_\phi)) \leq \bnu$. The fact that $\stheta$ is always in the preserved halfspace $\vH_{\phi}(\vhst_\phi,\omegast_\phi)$ follows directly from Lemma~\ref{lem:protect-stheta}. \end{proof}

\subsection{Proof of Theorem~\ref{thm:agnostic-corpv}}\label{ssec:computation}\label{ssec:analysis_adv_irrationality}
We now provide the guarantee for the $\bc$-known-corruption setting, whose proof is in Appendix~\ref{app:key_proposition}. Before delving into the details, we make two remarks. First, the regret guarantee of Proposition~\ref{prop:known-c} is \emph{deterministic}; only the runtime is randomized. Second, although the expected runtime is exponential in $\bc$, the algorithm is eventually run with $\bc\approx \log(T)$, which renders it quasipolynomial. 
\begin{proposition}\label{prop:known-c}
For the $\bc$-known-corruption setting, the regret of \textsc{CorPV.Known} for the $\eps$-ball loss is $\calO \left( (d^2 \bc + 1) d \log \left(\nicefrac{d}{\eps} \right)\right)$. When run with parameter $\eps=1/T$, its guarantee for the absolute and pricing loss is $\calO \left( (d^2 \bc + 1) d \log \left(dT \right)\right)$. The expected runtime is $\calO( (d^2 \bc )^{\bc} \cdot \poly (d \log ( \nicefrac{d}{\eps} ), \bc ) )$
\end{proposition}

\xhdr{Runtime of \textsc{CorPV.SeparatingCut} (Algorithm~\ref{algo:separating_cut}) .} The first step is to analyze \textsc{CorPV.SeparatingCut}  (Lemma~\ref{lem:computation}). For what follows, let $\calB_{L_\phi}(\vp^{\star}, \zeta)$ be the ball of radius $\zeta$ around $\vp^{\star}$ in the space of large dimensions, where $\vpst \in \Land_\phi$ is the landmark such that $u_\phi(\vpst,\nu) = \bc \cdot (d+1)+1$. Recall that we proved the existence of a landmark $\vpst\in\Land_\phi$ with this property in Lemma~\ref{lem:landmark3}. 

\begin{lemma}\label{lem:computation}
For any epoch $\phi$, scalar $\delta \in  (0,\frac{\eps}{2\sqrt{d} + 4d})$, and scalar $\bnu = \frac{\eps - 2 \sqrt{d} \cdot \delta}{4\sqrt{d}}$, after $\tau =2 d \cdot \bc (d+1) + 1$ rounds, algorithm \textsc{CorPV.SeparatingCut} computes hyperplane $(\tvh_{\phi}, \tomega_\phi)$ \emph{orthogonal to all small dimensions $S_\phi$} such that $\dist (\vkappast_\phi, (\tvh_\phi, \tomega_\phi))\leq 3\bnu$, and the resulting halfspace $\vH^+(\tvh_\phi,\tomega_\phi)$ always contains $\stheta$. With probability \emph{at least} $(40 d \cdot \sqrt{d-1})^{-1}$ the complexity of this computation is: \[\calO \left(\frac{(d-1)}{\bnu^2}\ \cdot \left( d^2 \cdot \bc \right)^{\bc} \cdot O(\text{CP}(d, \bc \cdot (2d(d+1) - 1) + 1))\right)\]
where $\text{CP}(n, m)$ is the complexity of solving a Convex Program with $n$ variables and $m$ constraints.
\end{lemma}

\xhdr{Bounding the number of epochs.} The second step is to establish that we make enough volumetric progress when using $(\tvh_\phi, \tomega_\phi)$ as our separating cut for epoch $\phi$. We remark that in the analysis of \cite{LLV18}, when \textsc{ProjectedVolume} observes a context $\vx_t$ such that $w(\Cyl(\calK_\phi, S_\phi),\vx_t) \leq \eps$, then it can directly discard it, since $\vx_t$ does not contribute to the regret with respect to the $\eps$-ball loss function. This is because $\vx_t$'s are used in order to make the separating cuts. In our epoch-based setting, the separating cuts are different than the observed contexts, as we have argued. Importantly, if $w(\Cyl(\calK_\phi,S_\phi), \tvh_\phi) \leq \eps$, we cannot relate this information to the regret of epoch $\phi$, because for all rounds comprising the epoch, the width of $\calK_\phi$ in the direction of the observed context was greater than $\eps$ (Step~\ref{step:geq-eps} of Algorithm~\ref{alg:corpv_known}). This is shown in the following lemma.

\begin{lemma}\label{lem:num-epochs}
After at most $\Phi = O(d \log (d/\eps))$ epochs, \textsc{CorPV.Known} (Algorithm~\ref{alg:corpv_known}) has reached a knowledge set $\calK_\Phi$ with width at most $\eps$ in \emph{every} direction $\vu$.
\end{lemma}

\xhdr{Extending to unknown corruption $C$.}
To turn Proposition~\ref{prop:known-c} to Theorem~\ref{thm:agnostic-corpv}, similar to \cite{LML18}, we separate the layers $j$ of Algorithm~\ref{algo:corpvAI} into corruption-tolerant ($j \geq \log C$) and corruption-intolerant ($j< \log C$). Since the corruption-tolerant layers, with high probability do not remove $\stheta$ from their parameter set, we view them as running independently for analysis purposes; each results to a regret equal to the one of Proposition~\ref{prop:known-c} with $\bc=\log T$. The corruption-intolerant layers may eliminate $\stheta$ but their knowledge set is eventually refined by the knowledge set of the first corruption-tolerant layer $\lceil  \log C\rceil$ thanks to global eliminations. Since the latter is selected with probability $1/C$ at every round, the time it will take for it to make volumetric progress is $C$ times what it would happen if it was run independently. The full proof is provided in  Appendix~\ref{app:analysis_agnostic}.

\xhdr{Relationship to Ulam's game.}
Ulam's game can be thought of as a non-contextual (1-dimensional) version of our problem with a known corruption level $C$ and the $\eps$-ball loss. In that setting, Rivest et al.~\cite{RMKWS90} show that the optimal query complexity $Q$ for localizing $\thetast$ to a region of volume $\eps$ satisfies $\eps \geq \sum_{i=0}^C {\binom{Q}{i}} \cdot 2^{-Q}$. The authors point out that this implies a query complexity lower bound of $\Omega(\log(1/\eps) + C\log\log(1/\eps) + C \log C)$. Beyond the fact that we consider the contextual setting, a subtle distinction between this analysis and ours is the difference between query complexity and regret. When measuring query complexity, we count every round until we can certify that we have localized $\thetast$, but the $\eps$-ball loss may be zero on rounds prior to this event. For example, consider an explore query at round $t$, such that $|\thetast - \omega_t| \leq \eps$; this query incurs an $\eps-$ball loss of $0$, but adds $1$ towards the query complexity count. As a result, the $\eps$-ball loss is always smaller than the query complexity in the non-contextual
setting. In the contextual setting, query complexity is not a meaningful metric as the adversary can inject many queries in directions that we have already learned without changing the problem. In particular, the algorithm incurs $0$ loss and does not use those queries despite not having yet estimated $\stheta$ in other directions. However, it is meaningful to define a notion of \emph{explore-query complexity} that counts the number of times that the algorithm either makes a mistake or uses the response of the round. Explore-query complexity also upper bounds the $\eps$-ball loss and our analysis actually bounds this notion. For this metric, the lower bound in~\cite{RMKWS90} suggests that a multiplicative relationship between $C$ and some function of $T$ (which appears in our bound) is unavoidable when $\eps = 1/T$. It is an interesting open question to understand whether this is the case for the regret notion that only penalizes the number of mistakes and for other loss functions especially because this multiplicative relationship is not present in multi-armed bandits~\cite{GKT19,ZimmertSeldin21}.

\xhdr{Discussion of algorithmic choices.} At this point, one would wonder whether \textsc{ProjectedVolume} has some particular special property that makes it amenable to our technique or whether we provided a generic reduction from any uncorrupted  contextual search algorithm.
It turns out that our approach relies on two properties of the uncorrupted algorithm: a) it needs to be “binary-search”, i.e., work with a knowledge set and refine it over time and b) separate the space in small and large dimensions. The latter is important as we do not make a cut on one of the existing contexts but rather combine them appropriately. As a result, an algorithm that works with projection on the large dimensions is always guaranteed to return a cut on that projection (therefore with sufficiently large width enabling volumetric progress). This is a property that is particular to \textsc{ProjectedVolume} and is not shared by other algorithms. We elaborate upon this discussion in Appendix~\ref{app:discussion}.

%% file: gradient_descent.tex
\section{Gradient descent algorithm}\label{sec:gradient_descent}
In this section, we propose our second algorithm, which is a variant of gradient descent and works for contextual search with absolute and $\eps$-ball loss. This algorithm is significantly simpler than algorithms based on binary search methods and has a better running time. On the other hand, it does not provide logarithmic guarantees when $C \approx 0$ and it does not extend to the pricing loss.

\begin{algorithm}[H]
\caption{$\gd$}
\label{algo:gd}
\DontPrintSemicolon
\SetAlgoNoLine
\SetAlgoNoEnd
Initialize $\vz_1 \in \calK$ and $\gamma_1 = 1/2$. \;
\For{rounds $t = \{1, \dots, T\}$}{
    For context $\vx_t$, query $\omega_t = \langle \vx_t, \vz_t \rangle$ and receive feedback: $y_t = \sgn(\omega_t - \langle \stheta, \vx_t \rangle)$. \;
    Choose $\vz_{t+1} = \Pi_{\calK} \left( \vz_t - \gamma_t \nabla f_t(\vz_t)\right)$, where $\gamma_t = \min\{1/2,\sqrt{2/t}\}$ and $f_t(\vz) = - y_t \cdot \langle \vz, \vx_t \rangle.$\;
}
\end{algorithm}

To explain the intuition behind Algorithm~\ref{algo:gd}, we restrict our attention to the absolute loss and recall that our goal is to minimize it using only binary feedback. The algorithm optimizes a \emph{proxy} function $f_t(\vz): \calK \to \bbR^d$, which is Lipschitz. Specifically, denoting the binary feedback by $y_t = \sgn (\omega_t - \langle \vx_t , \stheta \rangle )$, the proxy function is $f_t(\vz) = -y_t \cdot \langle \vx_t, \vz \rangle$. The query point at the next round $t+1$ is $\omega_{t+1} = \langle \vx_{t+1}, \vz_{t+1} \rangle$. Note here that $y_t$ is the subgradient of the target loss function $|\langle \stheta - \vz, \vx_t\rangle|$. The proxy function $f_t(\vz)$ is convenient because on the one hand, it is Lipschitz and on the other, its regret is an \emph{upper bound} on the regret incurred by any algorithm optimizing the absolute loss for the same problem. Additionally, in the presence of adversarial corruptions,
the \emph{same} algorithm suffers regret $\calO(\sqrt{T} + C)$; this is due to the fact that adversarial corruptions only add an extra set of $C$ erroneous rounds, from which the algorithm can certainly ``recover'' as there is no notion of a shrinking knowledge set. The proof of the following result is provided in Appendix~\ref{app:GD}.

\begin{theorem}\label{thm:GD-FR}
For an unknown corruption level $C$, $\gd$ incurs, in expectation, regret $\calO(\sqrt{T} + C)$ for the absolute loss and $\calO(\sqrt{T}/\eps + C/\eps)$ for the $\eps$-ball loss.
\end{theorem}

%% file: conclusion.tex
\section{Conclusion}
In this paper, we initiated the study of contextual search under adversarial noise models, motivated by pricing settings where some agents may be \emph{adversarially corrupted} and act in ways that are inconsistent with respect to the underlying ground truth. Although classical algorithms may be prone to even a few such agents, we show two algorithms that achieve near-optimal (uncorrupted) regret guarantees, while degrading gracefully with the number $C$ of corrupted agents.

Our work opens up many fruitful avenues for future research. First, the regret in both of our algorithms is sublinear when $C=o(T)$ but becomes linear when $C=\Theta(T)$. Designing algorithms that can provide sublinear regret against the ex-post best linear model, in the latter regime, is an exciting direction of future research and our model offers a concrete formulation of this problem. Second, our algorithm that attains the logarithmic guarantee has a regret of the order of $\calO(Cd^3\poly\log T)$. It would be interesting to either refine our approach or provide new algorithms that improve the dependence on $d$ and also remove the dependence on $T$ for the absolute and $\eps$-ball loss where such guarantees exist in the uncorrupted case. After a sequence of papers, the dependence in both fronts is now optimized when all agents are fully rational \citep{CLL16,LLV18,LS18,LPLS20}. Finally, we note that our first algorithm has quasi-polynomial running time; it is an intriguing open question to provide a polynomial-time algorithm that enjoys logarithmic guarantee when $C\approx 0$ for the loss functions we study.

%% file: app_glossary.tex
\section{Glossary}\label{app_glossary}

\begin{center}
\begin{tabular}{|c|c|} 
 \hline
 {\bf Notation} & {\bf Explanation} \\
 \hline \hline
 $\calK$ & parameter space \\ 
 \hline
 $w(\calK_\phi, \vx)$ & width of convex body $\calK_\phi$ in the direction of $\vx$ \\
 \hline
 $\thetast$ & ground truth common feature value \\ \hline 
 $C$ & true (unknown) number of corruptions \\ \hline
 $\bc$ & known number of corruptions (analysis only) \\ \hline
 $\vkappa_\phi$ & approximate centroid of knowledge set $\calK_\phi$ \\
 \hline
 $\vkappast_\phi$ & centroid of knowledge set $\calK_\phi$ \\
 \hline
 $L_\phi$ & set of large dimensions of epoch $\phi$ \\ 
 \hline
 $S_\phi$ & set of small dimensions of epoch $\phi$ \\ 
 \hline
 $\Cyl(\calK_\phi, S_\phi)$ & cylindrification of $\calK_\phi$ with small dimensions $S_\phi$ \\ 
 \hline
 $\calA_\phi$ & set of explore queries happening within epoch $\phi$\\ 
 \hline
 $(\vh, \omega)$ & hyperplane with normal vector $\vh$ and intercept $\omega$ \\
 \hline
 $\Pi_{L_\phi} \calK_\phi$  & projection of $\calK_\phi$ in large dimensions $L_\phi$ \\ 
 \hline
 $\vH^+(\vh, \omega)$ (resp. $\vH^-(\vh, \omega)$)  & positive (resp. negative) halfspace defined by $(\vh, \omega)$ \\ 
 \hline
 $\Lambda_\phi$ & set of landmarks of epoch $\phi$ \\ 
 \hline
 $\tau$  & length of epoch in rounds \\
 \hline
 $u_\phi(\vp, \nu)$  &  $\nu$-margin projected undesirability level\\
 \hline
 $\calP(\bc, \nu)$  &  $\bc$-protected region in large dimensions\\
 \hline    
 $\vpst$ & landmark that is used for the separating cut \\
 \hline
\end{tabular}
\end{center}

%% file: app_related.tex
\section{Further related work}\label{app:related}

\xhdr{Second methodological approach.} The first guarantees for this contextual setting can be traced to the work of Goldenshluger and Zeevi \cite{GoldenshlugerZeevi13} who relied on least squares estimation. Subsequently, Bastani and Bayati \cite{BastaniBayati20} showed how to incorporate sparsity in the guarantees using an appropriately designed LASSO estimator; the dependence on the sparsity parameter was further improved by Javanmard and Nazerzadeh \cite{JN19}. Qiang and Bayati \cite{QB16} consider a richer feedback setting where one can observe the demand at a particular price rather than the binary feedback we consider in this work. This model was extended in two directions: Nambiar, Simchi-Levi, and Wang \cite{NSLW19} allow for misspecifications in the demand model that the learner uses, while Ban and Keskin \cite{BK17} incorporate sparsity in the guarantees. The latter work provides guarantees for binary feedback under a parametric noise distribution, while Shah, Blanchet, and Johari \cite{SBJ19} extend these to non-parametric distributions. Finally, there has been work on the pure estimation side of the problem (without the decision-making component); e.g., Chen, Owen, Pixton, and Simchi-Levi \cite{StatLearningPersonal} provide convergence rates on estimation of $\stheta$ that depends on the sequence of previously selected prices (treating the latter as exogenous).

\xhdr{Dealing with non-i.i.d. rewards beyond adversarial corruptions.} Apart from adversarial corruptions, there have also been other models in the multi-armed bandit literature that go beyond dealing with i.i.d. rewards. Classical adversarial multi-armed bandit algorithms such as EXP3 \citep{EXP3} make no assumption on the reward sequence (they can come from an adaptive adversary) and compare against the best arm in hindsight. A relevant line of work is the one of \emph{best of both worlds} that aims to design a single algorithm that achieves the optimal stochastic (logarithmic) guarantees when the input is i.i.d. while also retaining the adversarial guarantees of EXP3 otherwise \citep{BubeckSli12}. The model of adversarial corruptions interpolates between these two extremes allowing to capture the \emph{middle ground} where most of the data are i.i.d. (or according to the dominant behavioral model) but some can behave arbitrarily and even adversarially. Another variant of EXP3, termed EXP3.S, is more tailored to dynamically evolving settings and compares against a stronger benchmark, i.e., the best sequence of arms that changes at most $S$ times; this is typically referred to as \emph{dynamic} or \emph{tracking} regret \citep{HerbsterWarmuth}. Subsequent works extend this setting by positing that the rewards come from distributions that can change over time either in a smooth manner or subject to a particular variation budget \citep{BrownianAlex, Gur1, Gur2, KeskinZeevi17, Cheungetal, COLT21Best}.

\xhdr{Non-contextual dynamic pricing.} Beyond the contextual setting, after the work of Kleinberg and Leighton \cite{KL03}, many papers incorporated important facets of dynamic pricing. Besbes and Zeevi \cite{BesbesZeevi09} provide learning policies based on maximum likelihood estimation when there exist inventory constraints and a finite horizon of interactions between seller and buyers. Broder and Rusmevichientong \cite{BR12} prove a regret rate of $\Theta(\sqrt{T})$ when the demands come from generic distributions without extra knowledge of their parameters, and an optimized regret of $\Theta (\log T)$ when the distributions are ``well-separable''. Moreover, den Boer and Zwart \cite{dBZ14} model the demand as a random variable, but the seller is assumed to only know the relationship between the first two moments and the selling price in order to construct the estimates of the optimal selling prices through averaging over shrinking intervals of the prices that have been chosen thus far.  Another setting studied by den Boer \cite{dB14} is the multi-product pricing setting, where only the first two moments of the demand distribution are known by the seller. Keskin and Zeevi \cite{KZ14} also study multi-product dynamic pricing when the demand function is linear and is perturbed by subgaussian noise, and they use least-squares linear regression techniques in order to estimate it. Roth, Slivkins, Ullman, and Wu \cite{RUW16,RothetalMultidimPricing} study a multi-unit, divisible pricing setting where the agent at each round purchases the bundle that maximizes their private utility function, and the learner observes only the purchased bundle (\emph{``revealed preferences''} feedback). Cesa-Bianchi, Cesari, and Perchet \cite{CBCP19} focus on a generalization of the stochastic problem where instead of assuming any smoothness on the distribution of the valuations, they assume that the distribution of the buyers' valuations is supported on an unknown set of unknown finite cardinality. Finally, Podimata and Slivkins~\cite{AdvZooming} extend the direction of fully adversarial valuations with an algorithm that enjoys regret better than $o(T^{2/3})$ for ``nicer'' instances and $\calO(T^{2/3})$ in the worst-case.

\xhdr{Ulam's game variants.} Ulam's game has also been studied for the case that the lies observed are a constant fraction of the total answers issued \citep{Pelc87, SpencerWinkler}, a constant fraction of the prefix of the answers are lies~\citep{AslamDhagat91}, and the target number is drawn from a known distribution \citep{Dagan2}. Our work is also related to works in noisy binary search with stochastic noise. Karp and Kleinberg~\cite{KK07} consider the setting where the learner is given $n$ biased coins $p_1\leq p_2\leq\ldots\leq p_n$, and the goal of the learner is to identify an interval $[p_{i'}, p_{i'+1}]$ that contains a given number. Nowak~\cite{Nowak08,Nowak09} studies Generalized Binary Search in which the learner wishes to identify a target function among a family of functions satisfying certain geometric properties, when the learner can only receive binary feedback under a stochastic noise model. Our work presents two main differences: we consider a contextual version of the problem (instead of non-contextual) and we consider an adversarial noise model (instead of stochastic noise).

%% file: app_projected_volume.tex
\section{Uncorrupted contextual search for \texorpdfstring{$\eps$}--ball loss}

\subsection{\textsc{ProjectedVolume} algorithm and intuition}\label{app:proj-vol}
In this subsection, we describe the \textsc{ProjectedVolume} algorithm of \cite{LLV18}, which is the algorithm that \textsc{CorPV.Known} builds on. \textsc{ProjectedVolume} minimizes the $\eps$-ball loss for fully rational agents by approximately estimating $\stheta$. At all rounds $t \in [T]$ \textsc{ProjectedVolume} maintains a convex body, called the \emph{knowledge set} and denoted by $\calK_t \in \bbR^d$, which corresponds to all parameters $\btheta$ that are not ruled out based on the information until round $t$. It also maintains a set of orthonormal vectors $S_t = \{\vs_1, \dots, \vs_{|S_t|}\}$ such that $\calK_t$ has small width along these directions, i.e., $\forall \vs \in S_t: w(\calK_t,\vs) \leq \delta'$. The algorithm ``ignores'' a dimension of $\calK_t$, once it becomes small, and focuses on the projection of $\calK_t$ onto a set $L_t$ of dimensions that are orthogonal to $S_t$ and have larger width, i.e., $\forall \vl \in L_t: w(\calK_t, \vl) \geq \delta'$.

\begin{algorithm}[H]
\caption{\textsc{ProjectedVolume} \citep{LLV18}}
\label{alg:projected_volume}
\DontPrintSemicolon
\SetAlgoNoLine
Initialize $S_0 \gets \emptyset, \calK_0 \gets \calK$. \;
\For{$t \in [T]$}{
 context $\vx_t$, chosen by nature. \;
Query point $\omega_t = \langle \vx_t, \vkappa_t \rangle$, where $\vkappa_t \gets \text{approx-centroid}(\Cyl(\calK_t,S_t))$. \;
Observe feedback $y_t$ and set $\calK_{t+1} \gets \calK_t \bigcap \vH^+(\vx_t, \omega_t)$ if $y_t=+1$ or $\calK_{t+1} \gets \calK_t \bigcap \vH^-(\vx_t, \omega_t)$ if $y_t=-1$. \;
Add all directions $\vu$ orthogonal to $S_t$ with $w(\calK_{t+1},\vu)\leq \delta' = \frac{\eps^2}{16d(d+1)^2}$ to $S_{t}$. \;
Set $S_{t+1} = S_t$.
}
\end{algorithm}

\noindent At round $t$, after observing $\vx_t$, the algorithm queries point $\omega_t = \langle \vx_t, \vkappa_t \rangle$ where $\vkappa_t$ is the \emph{approximate} centroid of knowledge set $\calK_t$. Based on the feedback, $y_t$, the algorithm eliminates one of $\vH^+(\vx_t, \omega_t)$ or $\vH^-(\vx_t, \omega_t)$. The analysis uses the volume of $\Pi_{L_t}\calK_t$, denoted by $\vol \left( \Pi_{L_t}\calK_t \right)$, as a potential function. After each query either the set of small dimensions $S_t$ increases, thus making $\vol \left( \Pi_{L_t} \calK_t \right)$ increase by a bounded amount (which can happen at most $d$ times), or $\vol \left( \Pi_{L_t} \calK_t \right)$ decreases by a factor of $\left(1 - \nicefrac{1}{e^2} \right)$. This potential function argument leads to a regret of at most $\bigO\prn*{d\log(\nicefrac{d}{\eps})}$.

%% file: app_single_dimension.tex
\subsection{Failure of \textsc{ProjectedVolume} against corruptions in one dimension}\label{app:single_dimension}
When $d=1$ and $\bc = 0$, there exists a $\stheta \in \bbR$ and nature replies whether $\omega_t$ is \emph{greater} or \emph{smaller} than $\stheta$. By appropriate queries, the learner can decrease the size of the knowledge set that is consistent with all past queries so that, after $\log(1/\eps)$ rounds, she identifies an $\eps$-ball containing $\thetast$. However, even when $\bc=1$, the above algorithm can be easily misled. Think about an example as in Figure~\ref{subfig:bin-search1} where $\stheta = 3/4$, the learner queries point $1/2$, and nature corrupts the feedback making her retain the interval $[0,1/2]$, instead of $[1/2, 1]$, as her current knowledge set.

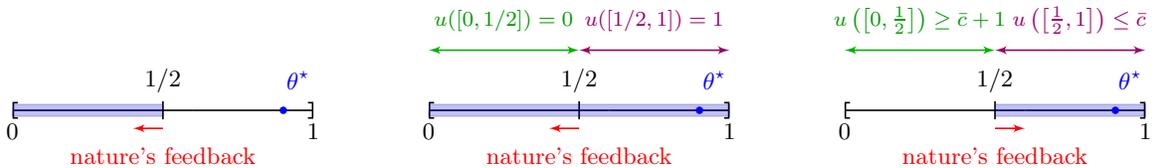
\begin{figure}[htbp]
\centering
\begin{subfigure}{0.33\textwidth}
\centering
\begin{tikzpicture}[line cap=timestep,line join=timestep,>=latex',x = 1cm, y = 1cm, scale=0.4]
\draw [color=white,, xstep=1cm,ystep=1cm] (-1, -2) grid (11,4.5);
\clip(-1,-2) rectangle (11,4.5);
\draw[[-, color=black, semithick] (0,0)  -- (6,0);
\draw[[-, color=black, semithick] (10,0) -- (6,0);
\draw[color=blue, fill=blue]      (9, 0) circle (3pt);
\draw[color=black, semithick]     (5, -0.3) -- (5, 0.3);
\draw[->, color=red, semithick]   (5, -0.6) -- (4, -0.6);
\draw[line width=0pt, fill=blue, opacity=0.25] (0,0) -- (0, -0.2) -- (5, -0.2) -- (5, 0.2) -- (0,0.2);
\begin{footnotesize}
\draw[color=black] (0,  -0.7) node {$0$};
\draw[color=black] (10, -0.7) node {$1$};
\draw[color=blue]  (9.5,  1 ) node {$\thetast$};
\draw[color=black] (5,    1 ) node {$1/2$};
\draw[color=red]   (5, -1.5 ) node {nature's feedback};
\end{footnotesize}
\end{tikzpicture}
\caption{\centering Binary search at round $1$.}\label{subfig:bin-search1}
\end{subfigure}\hfill
\begin{subfigure}{0.33\textwidth}
\centering
\begin{tikzpicture}[line cap=timestep,line join=timestep,>=latex',x = 1cm, y = 1cm, scale=0.4]
\draw [color=white,, xstep=1cm,ystep=1cm] (-1, -2) grid (11,4.5);
\clip(-1,-2) rectangle (11,4.5);
\draw[[-, color=black, semithick] (0,0)  -- (6,0);
\draw[[-, color=black, semithick] (10,0) -- (6,0);
\draw[color=blue, fill=blue]      (9, 0) circle (3pt);
\draw[color=black, semithick]     (5, -0.3) -- (5, 0.3);
\draw[->, color=red, semithick]   (5, -0.6) -- (4, -0.6);
\draw[line width=0pt, fill=blue, opacity=0.25] (0,0) -- (0, -0.2) -- (10, -0.2) -- (10, 0.2) -- (0,0.2);
\draw[<->, color=black!30!green, semithick]   (0,2) -- (5,2);
\draw[<->, color=blue!40!red, semithick] (5,2) -- (10,2);
\begin{footnotesize}
\draw[color=black] (0,  -0.7) node {$0$};
\draw[color=black] (10, -0.7) node {$1$};
\draw[color=blue]  (9.5,  1 ) node {$\thetast$};
\draw[color=black] (5,    1 ) node {$1/2$};
\draw[color=red]   (5, -1.5 ) node {nature's feedback};
\end{footnotesize}
\begin{scriptsize}
\draw[color=black!40!green] (2.5, 3) node {$u([0,1/2]) = 0$};
\draw[color=blue!40!red]    (7.5, 3) node {$u([1/2,1]) = 1$};
\end{scriptsize}
\end{tikzpicture}
\caption{\centering Undesirabilities at round $1$.}\label{subfig:bin-search2}
\end{subfigure}\hfill
\begin{subfigure}{0.33\textwidth}
\centering
\begin{tikzpicture}[line cap=timestep,line join=timestep,>=latex',x = 1cm, y = 1cm, scale=0.4]
\draw [color=white,, xstep=1cm,ystep=1cm] (-1, -2) grid (11,4.5);
\clip(-1,-2) rectangle (11,4.5);
\draw[[-, color=black, semithick] (0,0)  -- (6,0);
\draw[[-, color=black, semithick] (10,0) -- (6,0);
\draw[color=blue, fill=blue]      (9, 0) circle (3pt);
\draw[color=black, semithick]     (5, -0.3) -- (5, 0.3);
\draw[->, color=red, semithick]   (5, -0.6) -- (6, -0.6);
\draw[line width=0pt, fill=blue, opacity=0.25] (5,0) -- (5, -0.2) -- (10, -0.2) -- (10, 0.2) -- (5,0.2);
\draw[<->, color=black!30!green, semithick]   (0,2) -- (5,2);
\draw[<->, color=blue!40!red, semithick] (5,2) -- (10,2);
\begin{footnotesize}
\draw[color=black] (0,  -0.7) node {$0$};
\draw[color=black] (10, -0.7) node {$1$};
\draw[color=blue]  (9.5,  1 ) node {$\thetast$};
\draw[color=black] (5,    1 ) node {$1/2$};
\draw[color=red]   (5, -1.5 ) node {nature's feedback};
\end{footnotesize}
\begin{scriptsize}
\draw[color=black!40!green] (2.5, 3) node {$u\left(\left[0,\frac{1}{2}\right]\right) \geq \bc+1$};
\draw[color=blue!40!red]    (7.8, 3) node {$u\left(\left[\frac{1}{2},1\right]\right) \leq \bc$};
\end{scriptsize}
\end{tikzpicture}
\caption{\centering Undesirabilities at round $2\bc + 1$.}\label{subfig:bin-search3}
\end{subfigure}
\caption{Single dimensional binary search. The opaque band is the knowledge set after each query.}\label{fig:one-dim-bin-search}
\end{figure}

That said, if the learner knows $\bc$ then, by repeatedly querying the same point, she can guarantee that if she observes $y_t = +1$ (resp. $y_t = -1$) for \emph{at least} $\bc+1$ times, then feedback $y_t = +1$ (resp. $y_t = -1$) is definitely consistent with $\stheta$. Hence, by repeating each query $2\bc+1$ times the learner can incur regret at most  $(2\bc + 1)\log (\nicefrac{1}{\eps})$. Unfortunately, in higher dimensions it is impossible for the learner to repeat the exact same query, as nature chooses different contexts.

%% file: app_analysis.tex
\section{Main proofs for Section~\ref{sec:corpv_analysis}}

\subsection{Proof of Proposition~\ref{prop:known-c}}\label{app:key_proposition}

\begin{proof}[Proof of Lemma~\ref{lem:computation}.]
From Lemma~\ref{lem:existence}, there exists a hyperplane $(\vhst_\phi,\omegast_\phi)$ with distance at most $\bnu$ from $\vkappa_{\phi}$ that has all of $\calP(\bc,\nu)$ inside $\vH^+(\vhst_\phi,\omegast_\phi)$. By arguing about the volume contained in any specified ``cap'' of a multi-dimensional ball, we prove that with probability at least $(20\sqrt{d-1})^{-1}$ we can identify a point $
\vq$ lying on the halfspace \[ \vH^+\left(\vhst_\phi, \left \langle \vhst_\phi, \vpst \right\rangle + \frac{\zeta \cdot \ln(3/2)}{\sqrt{d-1}} \right).\] This is formally stated and proved in Lemma~\ref{lem:cap-vol} (Appendix~\ref{app:computation}). In each
iteration of \textsc{CorPV.SeparatingCut} using point $\vq$ (i.e., lines~\ref{step:perc-mb}--\ref{step:check-calpc} in Algorithm~\ref{algo:separating_cut}), Perceptron can identify a hyperplane $(\tvh_{\phi}, \tomega_\phi)$ separating $\vq$ and $\calP(\bc,\nu)$ after $\frac{d-1}{\zeta^2\cdot \ln^2(3/2) }$ samples. The number of samples depends on the Perceptron mistake bound, as mentioned in Section~\ref{ssec:algorithmic_crux}. Since $\vq \in \calB_{L_\phi}(\vpst,\zeta)$, then, 
\begin{align*}
\|\vq- \vkappast_\phi \| &= \|\vq - \vkappa_\phi + \vkappa_\phi - \vkappast_\phi \| \\  
&\leq \|\vq - \vkappa_\phi\| + \|\vkappa_\phi - \vkappast_\phi \| &\tag{triangle inequality}\\
&\leq \|\vq - \vkappa_\phi\| + \bnu &\tag{approximation of $\vkappast_\phi$ in polynomial time}\\ 
&= \|\vq - \vpst + \vpst - \vkappa_\phi \| + \bnu\\
&\leq \|\vq - \vpst \| + \|\vpst - \vkappa_\phi \| + \bnu &\tag{triangle inequality}\\
&\leq \bnu + \bnu + \bnu &\tag{$\vq \in \calB_{L_\phi}(\vpst, \zeta)$ and Definition~\ref{def:landmark}}
\end{align*} 
Hence, $\dist(\vkappast_\phi, (\tvh_\phi,\tomega_\phi)) \leq 3\bnu$.

Assume for now that the random landmark $\vpst \in \Land_\phi$ chosen at Step~\ref{step:sample-ball} of Algorithm~\ref{algo:separating_cut} is the landmark such that $u_\phi(\vpst, \nu) \geq \bc \cdot (d+1) +1$. Because $\vpst$ is chosen uniformly at random from set $\Land_\phi$, then the probability of it being the target landmark is $(2d)^{-1}$, and is independent with all other random variables of Algorithm~\ref{algo:separating_cut}.

At every iteration of the inner while-loop of Algorithm~\ref{algo:separating_cut}, the algorithm checks whether $\vq$ and $\vkappa_\phi$ are on the ``correct'' side of $\vh$ and possibly updates the Perceptron (this is done in time $\bigO(1)$): the centroid $\vkappa_\phi$ needs to be part of the $\calP(\bc, \nu)$, while $\vq$ needs to be separated from $\calP(\bc,\nu)$. Thus, they must belong to $\vH^+(\tvh, \tomega)$ and $\vH^-(\tvh,\tomega)$ respectively. This is done in Steps~\ref{step:check-centr-pt1}--\ref{step:check-centr-pt2} of Algorithm~\ref{algo:separating_cut}. The most computationally demanding part of Algorithm~\ref{algo:separating_cut} is Steps~\ref{step:comb1}--\ref{step:lp2}, where we check whether the hyperplane cuts some part of $\calP(\bc,\nu)$. For that, we check all the possible \[{\binom{|\calA_\phi|}{|D_\phi|}}  ={\binom{|\calA_\phi|}{\bc}} \leq (2d \cdot \bc (d+1)+1)^{\bc}=\Theta \left(\left(d^2\bc\right)^{\bc} \right)\]
combinations of which $\bc$ hyperplanes to disregard as corrupted and solve
a mathematical program which we refer to as CP~\footnote{Technically this program is convex and not linear as we also take intersection with the $\calK_\phi$ which is a convex body.} for each such combination. We remark that this computation serves the purpose of identifying which $\bc$ hyperplanes in $\mathcal{A}_\phi$
were corrupted, thus giving
erroneous feedback regarding where $\stheta$ lies. These CPs have at most $d$ variables (since $\calK_\phi \subseteq \bbR^d$) and $|\calA_\phi| - \bc = \bc \cdot (2d(d+1) - 1) + 1)$ constraints. Denoting the complexity of solving a CP with $n$ variables and $m$ constraints as $O(\text{CP}(n,m))$ we therefore obtain that Steps~\ref{step:comb1}--\ref{step:lp2} have computational complexity \[\calO\left( O(\text{CP}(d,\bc \cdot (2d(d+1) - 1) + 1))) \cdot \left( d^2 \bc \right)^{\bc} \right).\] 

Putting everything together, and given that the event that $\vpst$ is the target landmark is independent from the event that $\tvq$ is found at the desired halfspace we have that with probability at least \[\frac{1}{2d} \cdot \frac{1}{20\sqrt{d-1}} \] the computational complexity of Algorithm~\ref{algo:separating_cut} is:
\[\calO \left(\frac{(d-1)}{\bnu^2}\ \cdot \left( d^2 \cdot \bc \right)^{\bc} \cdot O(\text{CP}(d, \bc \cdot (2d(d+1) - 1) + 1))\right)\]
This concludes our proof.
\end{proof}

\begin{proof}[Proof of Lemma~\ref{lem:num-epochs}.]
To prove this lemma, we construct a potential function argument, similar to the one of \cite{LLV18} and we highlight the places where our analysis differs from theirs.

We use $\Gamma_\phi = \vol\left( \Pi_{L_\phi} \calK_\phi\right)$ as our potential function. Its lower bound is: $\Gamma_\phi \geq \Omega \left( \frac{\delta}{d} \right)^{2d}$. To see this, note that Step~\ref{step:keep-large} of \textsc{CorPV.EpochUpdates} (Algorithm~\ref{alg:corpv_epoch_updates}) ensures that for all $\vu \in L_\phi$ it holds that $w\left(\Pi_{L_\phi} \calK_\phi, \vu \right) \geq \delta$. It is known (see \cite[Lemma~6.3]{LLV18} or Lemma~\ref{lem:ball-inside}) that if $\calK \subseteq \bbR^d$ is a convex body such that $w(\calK, \vu) \geq \delta$ for every unit vector $\vu$, then $\calK$ contains a ball of diameter $\delta/d$. This means that $\Pi_{L_\phi} \calK_\phi$ contains a ball of radius $\frac{\delta}{|L_\phi|}$, so \[ \vol\left( \Pi_{L_\phi}\calK_\phi\right) \geq V(\abs{L_\phi}) \left( \frac{\delta}{|L_\phi|} \right)^{|L_\phi|},\] where by $V(\abs{L_\phi})$ we denote the volume of the $|L_\phi|$-dimensional unit ball. Using that $|L_\phi| \leq d$ and $V(d) \geq \Omega\left(\frac{1}{d} \right)^d$, the latter can be lower bounded by $\Omega \left(\frac{\delta}{d} \right)^{2d} $. Hence, $\Gamma_\phi = \vol\left( \Pi_{L_\phi} \calK_\phi\right) \geq \Omega \left(\frac{\delta}{d} \right)^{2d} $.

We split our analysis of the upper bound of $\Gamma_\phi$ in two parts. In the first part, we study the potential function between epochs where the set of large dimensions $L_\phi$ does not change. In the second part, we study the potential function between where the set of large dimensions $L_\phi$ becomes smaller. For both cases, we prove the following useful result (Lemma~\ref{lem:epoch-grun}) which relates the volume of $\Pi_{L_\phi} \calK_{\phi+1}$ with the volume of $\Pi_{L_\phi} \calK_\phi$ when $\delta = \frac{\eps}{4 (d + \sqrt{d})}$:
\begin{equation}\label{eq:vol-same-large}
\vol \left( \Pi_{L_{\phi}} \calK_{\phi+1} \right) \leq \left( 1 - \frac{1}{2e^2} \right) \vol \left( \Pi_{L_{\phi}} \calK_\phi \right)
\end{equation}
When the set $L_\phi$ does not change between epochs $\phi, \phi+1$ (i.e., $L_\phi = L_{\phi+1}$) Equation~\eqref{eq:vol-same-large} becomes: 
\begin{equation}\label{eq:vol-same-large2}
\vol \left( \Pi_{L_{\phi+1}} \calK_{\phi+1} \right) \leq \left( 1 - \frac{1}{2e^2} \right) \vol \left( \Pi_{L_{\phi}} \calK_\phi \right)
\end{equation}
When $L_\phi$ does change, then the set of small dimensions increases from $S_\phi$ to $S_{\phi+1}$. In order to correlate $\vol ( \Pi_{L_{\phi+1}} \calK_{\phi+1} )$ with $\vol(\Pi_{L_\phi} \calK_\phi)$ we make use of the following known inequality (\cite[Lemma~6.1]{LLV18} or Lemma~\ref{lem:cyl}) that for a convex body $\calK \subseteq \bbR^d$ if $w(\calK, \vu) \geq \delta'$ (for some scalar $\delta' >0$, for \emph{every} unit vector $\vu$), then, for every $(d-1)-$dimensional subspace $L$ it holds that: 
\begin{equation}\label{eq:cyl}
    \vol \left( \Pi_L \calK \right) \leq \frac{d (d+1)}{\delta'} \vol(\calK)
\end{equation}
We are going to apply Equation~\eqref{eq:cyl} for $\calK := \Pi_{L_\phi} \calK_{\phi+1}$ and $L := L_{\phi+1}$. For that, we need to find $\delta'$ for which $w(\Pi_{L_\phi}\calK) \geq \delta'$.

We make use of the following lemma (\cite[Theorem~5.3]{LLV18} or Lemma~\ref{lem:dir-grun}) which relates the width of the following convex body: $\calK_{+} = \calK \bigcap \{\vx| \langle \vu, \vx - \vkappa \rangle =0\}$ (where $\vkappa$ is the centroid of $\calK$ and $\vu$ is any unit vector) with the width of $\calK$ in the direction of any unit vector $\vv$ as follows: 
\begin{equation}\label{eq:dir-grun}
\frac{1}{d+1} w(\calK,\vv) \leq w(\calK_+, \vv) \leq w(\calK,\vv) 
\end{equation}
By the definition of large dimensions, $w(\calK_\phi,\vu) \geq \delta, \forall \vu \in L_{\phi}$. So, if we were to cut $\calK_\phi$ with a hyperplane that passes precisely from the centroid $\vkappa_\phi^{\star}$ then, from Equation~\eqref{eq:dir-grun}: $w(\calK_+,\vu) \geq \frac{\delta}{d+1}$. Since, however, we make sure that we cut $\calK_\phi$ through the approximate centroid $\vkappa_\phi$, then $w(\calK_+,\vu) \geq \frac{\delta}{d+1}$. This is due to the fact that the halfspace $\vH^+(\tvh_\phi,\tomega_\phi)$ returned from \textsc{CorPV.SeparatingCut} (Algorithm~\ref{algo:separating_cut}) always contains $\vkappa_\phi$. Since $\|\vkappa_\phi - \vkappa_\phi^{\star}\| \leq \bnu$ then $\vkappa_\phi^{\star} \in \vH^+(\tvh_\phi,\tomega_\phi)$. As a result, from Equation~\eqref{eq:cyl} we get:
\begin{equation}\label{eq:vol-incr}
    \vol \left( \Pi_{L_{\phi+1}} \calK_{\phi+1}\right) \leq \frac{d(d+1)^2}{\delta} \vol \left( \Pi_{L_\phi} \calK_{\phi+1} \right)
\end{equation}
We have almost obtained the target Equation~\eqref{eq:vol-same-large}. To complete the argument we need to correlate $\vol(\Pi_{L_\phi} \calK_{\phi+1})$ with $\vol(\Pi_{L_\phi} \calK_\phi)$. We prove the following inequality between the two (Lemma~\ref{lem:epoch-grun}):
\begin{equation}\label{eq:almst}
\vol\left(\Pi_{L_{\phi}}\calK_{\phi+1}\right)\leq \left(1-\frac{1}{2e^2}\right)\vol(\Pi_{L_{\phi}}\calK_{\phi})
\end{equation}
Combining Equations~\eqref{eq:vol-incr} and~\eqref{eq:almst} and using the fact that we add at most $d$ new directions to $S_\phi$, that the volume of $\calK_0$ is upper bounded by $O(1)$, and the lower bound for $\Gamma_\phi$ we computed, we have: 
\begin{equation*}
    \Omega \left( \frac{\delta}{d} \right)^{2d} \leq \Gamma_\Phi= \vol \left(\Pi_{L_\Phi} \calK_\Phi \right) \leq O(1) \cdot \left(\frac{d(d+1)^2}{\delta} \right)^d \cdot \left( 1 - \frac{1}{2e^2} \right)^{\Phi}
\end{equation*}
where by $\Phi$ we denote the total number of epochs. Solving the above in terms of $\Phi$ and substituting $\delta$ we obtain: $\Phi \leq O ( d \log \frac{d}{\eps})$.
\end{proof}

\begin{proof}[Proof of Proposition~\ref{prop:known-c}]
We start by analyzing the regret for the $\eps$-ball loss function. Lemma~\ref{lem:num-epochs} establishes that after $\Phi = O (d \log (d/\eps))$ epochs, the set of large dimensions $L_\phi$ is empty. When $L_\phi = \emptyset$, then $S_\phi$ must be an orthonormal basis for which $w(\calK_\phi, \vs) \leq \delta, \forall \vs \in S_\phi$. For any received context $\vx_t$ after $L_\phi = \emptyset$, we have the following. First, $\vx_t = \sum_{i \in [|S_\phi|]} a_i \cdot s_i$, and for any vector $\mathbf{a}$ it holds that: $\|\mathbf{a}\|_1 \leq \sqrt{d} \cdot \|\mathbf{a}\|_2$. So, the width of $\calK_\phi$ in the direction of $\vx_t$ when $|S_\phi| = d$ is: 
\begin{align*}
w\left(\calK_\phi,\vx_t\right) &= \max_{\vp,\vq \in \calK_\phi} \langle \vx_t, \vp - \vq \rangle &\tag{definition of width}\\
&\leq \sum_{i \in [|S_\phi|]} a_i \cdot \max_{\vp,\vq \in K_\phi} \langle s_i, \vp - \vq \rangle &\tag{$\vx_t = \sum_{i \in [|S_\phi|]} a_i \vs_i$ and properties of $\max(\cdot)$}\\
&\leq \sum_{i \in [|S_\phi|]} a_i\cdot \delta &\tag{definition of small dimensions}\\
&\leq \|\mathbf{a}\|_1 \cdot \delta \leq \sqrt{d} \cdot \delta \cdot \|\mathbf{a}\|_2 &\tag{$\|\mathbf{a}\|_1 \leq \sqrt{d} \cdot \|\mathbf{a}\|_2$}
\end{align*}
Substituting $\delta = \frac{\eps}{4(d + \sqrt{d})}$ the latter becomes: $w(\calK_\phi, \vx_t) \leq \frac{\eps}{4(\sqrt{d}+1)}\leq \eps$. Therefore, when $L_\phi = \emptyset$, then \textsc{CorPV.Known} incurs no additional regret for any context it receives in the future, if we are interested in the $\eps$-ball loss. Using the fact that each epoch contains $\tau = 2d \cdot \bc (d+1) + 1$ rounds during which we can incur a loss of at most $1$ we have that the regret for the $\eps$-ball loss is equal to: 
\[R_{\eps-\text{ball}}(T) = O \left(d \log \frac{d}{\eps} \right) \cdot \left(2 d \cdot \bc \cdot (d+1) + 1  \right) = \calO \left((d^2 \bc + 1) d \log \left(\frac{d}{\eps} \right) \right) \]
For the absolute and the pricing loss, for every round after the set $L_\phi$ has become empty, the queried point incurs a loss of at most $\eps$. As a result the regret for both cases is \emph{at most}
\[R_{\eps-\text{ball}}(T) + \eps \cdot T\]
Tuning $\eps = 1/T$ we get the result for these two loss functions.
\end{proof}

\subsection{Proof of Theorem~\ref{thm:agnostic-corpv}}\label{app:analysis_agnostic}

Before providing the proof of Theorem~\ref{thm:agnostic-corpv}, we give some auxiliary probabilistic lemmas that we use.

\begin{lemma}[{\cite[Lemma~3.3]{LML18}}]\label{lem:cor-per-layer}
For corruption level $C$, each layer $j \geq \log C$ observes at most $\ln \left(\nicefrac{1}{\beta} \right) + 3$ corruptions with probability at least $1 - \beta$.
\end{lemma}

\begin{lemma}\label{app:union-bound}
Let $X_1, \dots, X_n$ denote $n$ random binary variables that take value of $0$ with probability \emph{at most} $p_1, \dots, p_n$ respectively. Then, the following is true: \[ \Pr \left[ \bigcap_{i \in [n]} X_i \right] \geq 1 - \sum_{i \in [n]} p_i\]
\end{lemma}

\begin{proof}
This inequality is proven using the union bound as follows: \[\Pr \left[ \bigcap_{i \in [n]} X_i \right] = 1 - \Pr [\exists j: X_j = 0] \geq 1 - \sum_{j \in [n]} p_j\]
\end{proof}

\begin{lemma}\label{lem:binom}
Let $X$ be a random variable following the binomial distribution with parameters $n$ and $p$, such that $p = 1/a$ for some $a > 0$. Then, $\Pr[X < 1] \leq \delta$ for $n = a \cdot \log (1/\delta)$. 
\end{lemma}
\begin{proof}
Using the definition of the binomial distribution we have that: $\Pr[X < 1] = \Pr[X = 0] = (1 - p)^n$. For any $\beta$ in order for the result to hold one needs 
\begin{equation}\label{eq:n}
n \geq \frac{\log(1/\beta)}{\log \left(\frac{1}{1-p} \right)}
\end{equation}
Since $\log \left(\frac{1}{1-p} \right) \geq \frac{p}{1-p}$ then Equation~\eqref{eq:n} is satisfied. Substituting $p  = 1/a$ we get the result.\end{proof}

\vspace{0.1in}
\begin{proof}[Proof of Theorem~\ref{thm:agnostic-corpv}]
We present the proof for the $\eps$-ball loss. Tuning $\eps = 1/T$ afterwards gives the stated result for the absolute and pricing loss.

We separate the layers into two categories: layers $j \geq \log C$ are \emph{corruption-tolerant}, and layers $j < \log C$ are \emph{corruption-intolerant}. Every layer $j$, if it were to run in isolation, would spend $\Phi_j$ epochs until converging to a knowledge set with width at most $\eps$ in all the directions. However, in \textsc{CorPV.AC} layer $j$'s epoch potentially gets increased every time that a layer $j' \geq j$ changes epoch. Since there are at most $\log T$ layers, this results in an added $\log T$ multiplicative overhead for the epochs of each layer. This overhead is suffered by the corruption-tolerant layers.

We first study the performance of the corruption-tolerant layers. Let $\beta_j > 0$ denote the failure probability for layer $j$ such that $\beta_j \leq \frac{\beta}{\log T+1}$. From Lemma~\ref{lem:cor-per-layer} we have that with probability at least $1 - \beta_j$, the actual corruption experienced by the tolerant layers is at most 
\begin{equation}\label{eq:tildeC}
\calC = \ln \left(\frac{1}{\beta_{j}}\right) + 3 \leq \log \left(\frac{T}{\beta}
\right).
\end{equation}
From the regret guarantee of Proposition~\ref{prop:known-c} for all rounds that this corruption-tolerant layer was sampled, the regret incurred by each tolerant layer $j$, denoted by $R_{\tol,j}$, is upper bounded by: 
\begin{equation}\label{eq:rell-tol}
R_{\tol,j} \leq \calO \left( \left(d^2 \calC + 1\right)d\log \left(\frac{d}{\eps} \right) \right)
\end{equation}

Since there are at most $\log T$ tolerant layers, then with probability at least $1 - \bar{\beta}$ (Lemma~\ref{app:union-bound}), where $\bar{\beta} = \sum_{j \in [\log T]} \beta_j = \frac{\log T}{\log T + 1} \beta$) the regret incurred by all the corruption-tolerant layers is: 
\begin{equation}\label{eq:tolerant}
    R_{\text{tolerant}} \leq \sum_{j = 1}^{\log T} R_{\tol,j}
\end{equation}

We now move to the analysis of the corruption-intolerant layers. Let $\jst$ denote the smallest corruption-tolerant layer, i.e., $\jst  = \min_j \left\{j \geq \log C \right\}$. Observe that each layer $j \leq \jst$ is played until layer $\jst$ identifies the target knowledge set having width at most $\eps$ in every direction. If $\jst$ was run in isolation, from Equation~\eqref{eq:rell-tol} it would incur regret $R_{\tol,\jst}$. When a context is not costly for $\jst$, it is also not costly for layers $j<\jst$. This follows because we have consistent knowledge sets and sets of small dimensions across the layers. As a result, whenever a context causes regret for a corruption-intolerant layer, with probability $1/C$, $\jst$ is selected and it makes progress towards identifying the target. Using standard arguments for the binomial distribution (see Lemma~\ref{lem:binom}) we can show that for any scalar $\tbeta >0$ with probability at least $1 - \tbeta$, layer $\jst$ is played \emph{at least once} every $N = C\log(1/\tbeta)$ rounds. Set $\tbeta$ to be $\tbeta \leq \beta/(\log T + 1)$. Hence, the total regret from corruption-intolerant layers can be bounded by the total regret incurred by the first corruption-tolerant layer times $N$. Mathematically:  
\begin{align*}
R_{\text{intolerant}} \leq N \cdot R_{\jst} 
& = \calO \left(N \cdot (2d(d+1) \log C + 1)d \log \left(\frac{d}{\eps}\right)\right) \\
&= \calO \left(C \cdot (2d(d+1) \log C + 1)d \log \left(\frac{d}{\eps}\right) \log \left(\frac{1}{\tbeta}\right)\right) \numberthis{\label{eq:intolerant}}
\end{align*}
until the appropriately small knowledge set is constructed for $\jst$; subsequently this knowledge set dictates the behavior of the intolerant layers. 

Putting everything together, and using the union bound again, we have that with probability at least $1 - \sum_{j \in [n]} \beta_j - \tbeta = 1 - \beta$ the regret of \textsc{CorPV.AC} is:
\begin{align*}
R &= R_{\text{tolerant}} + R_{\text{intolerant}} &\tag{Equations~\eqref{eq:rell-tol} and~\eqref{eq:intolerant}}\\
&\leq \calO \left( \left(\log T+C \right)\cdot \left(d^2 \calC + 1\right)d\log \left(\frac{d}{\eps} \right) \cdot \log \left(\frac{1}{\beta} \right) \right) \\&\leq \calO \left( d^3 \cdot \log \left(\min \left\{T, \frac{d}{\eps}\right\} \cdot \frac{1}{\beta} \right) \cdot \log \left(\frac{d}{\eps} \right) \cdot \log \left( \frac{1}{\beta}\right) \cdot \left(\log(T) + C \right)\right)
\end{align*}
We finally discuss the computational complexity of \textsc{CorPV.AC}. Note that the complexity is dictated by the choice of $\calC=\text{poly}\log(T))$. As a result, from Lemma~\ref{lem:computation}, substituting $C$ with $\log T$, we get that \textsc{CorPV.AC} has expected runtime: $\tilde{\calO}\left(\left(d^2 \log T \right)^{\text{poly}\log T} \cdot \text{poly}\left(d\log \frac{d}{\eps}, \log T\right)\right)$.
\end{proof}

\section{Auxiliary lemmas for Section~\ref{sec:corpv_analysis}}\label{app:known-c}

\subsection{Auxiliary lemmas for Lemma~\ref{lem:existence}
}\label{app:existence}

\begin{lemma}\label{lem:undesir-gt}
If $\nu > \unu$ (where $\unu=\sqrt{d}\delta$) then, for any point $\vp$, we have that:
$$u_\phi(\vp,\nu)=\sum_{t\in [\tau]}\1\left\{\left(\left\langle\vp-\vkappa_{\phi},\Pi_{L_{\phi}}\vx_t\right\rangle+\nu\right) \cdot \left( \left\langle \ttheta-\vkappa_{\phi},\Pi_{L_{\phi}}\vx_t\right \rangle+\nu\right)<0\right\}$$
\end{lemma}
\begin{proof}
In order to prove the lemma, we argue that $\langle \ttheta - \vkappa_\phi, \Pi_{L_\phi} \vx_t \rangle + \nu> 0$. Recall that the feedback in Step~\ref{step:feedback} of the protocol is defined as $\sgn(\tv_t - \omega_t)$, and that we set $y_t=+1$ and $\tv_t=\langle \vx_t,\theta_t\rangle$. As a result,
$\langle \ttheta - \vkappa_\phi, \vx_t \rangle \geq 0$ and expanding:
\begin{equation}\label{eq:lb-nu}
\langle \ttheta - \vkappa_\phi, \Pi_{L_\phi} \vx_t \rangle + \langle \ttheta - \vkappa_\phi, \Pi_{S_\phi} \vx_t \rangle \geq 0
\end{equation}
We proceed by upper bounding the quantity $\langle \ttheta - \vkappa_\phi, \Pi_{S_\phi} \vx_t \rangle$. Let $S$ be a matrix with columns corresponding to the basis of vectors in $S_\phi$, so that $\Pi_{S_\phi} = SS^\top$. Then, we obtain:
\begin{align*}
\langle \Pi_{S_\phi} \vx_t, \vp-\vkappa_\phi\rangle &= \langle \Pi_{S_\phi} \vx_t, \Pi_{S_{\phi}}(\vp-\vkappa_\phi)\rangle \leq  |\langle \Pi_{S_\phi} \vx_t, \Pi_{S_\phi}(\vp-\vkappa_\phi) \rangle| \\
&\leq \|\Pi_{S_\phi} \vx_t\|_2 \cdot \|\Pi_{S_\phi}(\vp-\vkappa_\phi)\|_2&\tag{Cauchy-Schwarz inequality}\\
& = \|\Pi_{S_\phi} \vx_t\|_2 \cdot \|S^\top(\vp-\vkappa_\phi)\|_2\\
&\leq \|\vx_t\|_2 \cdot \sqrt{d} \cdot \|S^\top (\vp-\kappa_\phi)\|_\infty  &\tag{$\|\vz\|_2 = \sqrt{\sum_{i \in [d]}\vz_i^2} \leq \sqrt{d\cdot \|\vz\|_{\infty}^2}$}\\
&\leq 1 \cdot \delta \sqrt{d}.&\tag{$\|\vx_t\|_2 = 1$ and $w(\calK_\phi, \vs) \leq \delta, \forall \vs \in S_\phi$}
\end{align*}
Using this to relax Equation~\eqref{eq:lb-nu} along with $\unu = \sqrt{d}\delta$ we get that: $\langle \ttheta - \vkappa_\phi, \Pi_{L_\phi} \vx_t \rangle \geq - \unu$. Since $\nu >\unu$, it follows that $\langle \ttheta - \vkappa_\phi, \Pi_{L_{\phi}} \vx_t\rangle + \nu \geq - \unu + \nu >0$. Combined with Definition~\ref{def:undesir-final}, this concludes the lemma.\end{proof}

\begin{proof}[Proof of Lemma~\ref{lem:protect-stheta}]
By Lemma~\ref{lem:undesir-gt}, $$u_\phi(\vp,\nu)=\sum_{t \in [\tau]}\1\left\{\left(\left\langle\vp-\vkappa_{\phi},\Pi_{L_{\phi}}\vx_t\right\rangle+\nu\right) \cdot \left( \left\langle \ttheta-\vkappa_{\phi},\Pi_{L_{\phi}}\vx_t\right \rangle+\nu\right)<0\right\}.$$
For the uncorrupted rounds $\stheta = \vp = \ttheta$; as a result, the corresponding summands are non-negative: $\left(\left\langle\vp-\vkappa_{\phi},\Pi_{L_{\phi}}\vx_t\right\rangle+\nu\right) \cdot \left(\left\langle \ttheta-\vkappa_{\phi},\Pi_{L_{\phi}}\vx_t\right \rangle+\nu\right)\geq 0$. Hence, the only rounds for which $\stheta$ can incur undesirability are the corrupted rounds, of which there are at most $\bc$. As a result, $u_{\phi}(\stheta,\nu)\leq \bc$ and  $\stheta \in \calP(\bc{,\nu})$ by the definition of region $\calP(\bc,\nu)$. \end{proof}

Before proving this result (formally in Lemma~\ref{lem:landmark2}) we need the following technical lemma, whose proof follows ideas from \cite{LLV18} and at the end of the section for completeness. 

\begin{lemma}\label{lem:landmark}
Let basis $E_{\phi} = \{\mathbf{e}_1, \dots, \mathbf{e}_{d - |S_\phi|}\}$ be orthogonal to $S_\phi$. For all $\{(\vx_t,\omega_t)\}_{t \in [\tau]}$ such that $w(\Cyl(\calK_\phi,S_\phi),\vx_t) \geq \eps$, there exists $i$ such that: $\left|\langle \mathbf{e}_i, \vx_t \rangle \right| \geq \bnu$, where $\bnu = \frac{\eps - 2 \sqrt{d} \cdot \delta}{4\sqrt{d}}$.
\end{lemma}
\noindent The tuning of $\bnu$ explains the constraint imposed on $\delta$, i.e., $\delta < \frac{\eps}{2\sqrt{d} + 4\delta}$. This constraint is due to the fact that since $\nu > \unu$ and $\nu < \bnu$, then it must be the case that $\unu < \bnu$, where $\unu = \sqrt{d}\delta$ and $\bnu = \frac{\eps - 2\sqrt{d}}{4\sqrt{d}}$. 

\begin{lemma}\label{lem:landmark2}
For every round $t \in [\tau]$, any scalar $\delta \in (0,\frac{\eps}{2\sqrt{d}+4d})$, any scalar $\nu < \bnu$, at least one of the landmarks in $\Land_\phi$ gets one $\nu$-margin projected undesirability point, i.e., 
\[\exists\vp \in \Land_{\phi}: \left(\left\langle\vp-\vkappa_{\phi},\Pi_{L_{\phi}}\vx_t\right\rangle+\nu\right)<0.\]
\end{lemma}

\begin{proof}
By Lemma~\ref{lem:landmark}, there exists a direction $\mathbf{e}_i \in E_{\phi}$ such that $|\langle \mathbf{e}_i, \vx_t \rangle | \geq {\bnu} = \frac{\eps - 2 \sqrt{d} \cdot \delta}{4\sqrt{d}}$. The proof then follows by showing that for $\nu < \bnu$ landmark points $\vq_+=\vkappa_{\phi}+ {\nu} \cdot \mathbf{e}_i$ and $\vq_-=\vkappa_\phi- {\nu} \cdot \mathbf{e}_i$ get different signs in the undesirability point definition. This is shown by the following derivation:
\begin{align*}
&\left(\left\langle\vq_+-\vkappa_{\phi},\Pi_{L_{\phi}}\vx_t\right\rangle+\nu\right) \cdot \left( \left\langle \vq_--\vkappa_{\phi},\Pi_{L_{\phi}}\vx_t\right \rangle+\nu\right) \\
&=\left(\left\langle {\bnu} \cdot \mathbf{e}_i,\Pi_{L_{\phi}}\vx_t\right\rangle+\nu\right) \cdot \left( \left\langle - {\bnu} \cdot \mathbf{e}_i,\Pi_{L_{\phi}}\vx_t\right \rangle+\nu\right)  \\
&= \nu^2  - \left({\bnu} \cdot |\langle \mathbf{e}_i, \vx_t \rangle|\right)^2 \leq \nu^2 - \bnu^2 < 0
\end{align*}
where the last inequality comes from the fact that $\nu \in (\unu,\bnu)$. As a result there exists $\vp\in\{\vq_+,\vq_-\}\subseteq \mathcal{L}_{\phi}$ satisfying the condition in the lemma statement.
\end{proof}

\begin{proof}[Proof of Lemma~\ref{lem:landmark3}]
At each of the $\tau$ explore rounds, at least one of the landmarks gets a $\nu$-margin projected undesirability point (Lemma~\ref{lem:landmark2}). Since there are \emph{at most} $2d$ landmarks, by the pigeonhole principle after $\tau$ rounds, there exists at least one of them with $\nu$-margin projected undesirability $u_\phi(\vp^{\star}, \nu) \geq \bc \cdot (d+1) +1$. Since all points $\vq$ inside $\conv(\calP(\bc{,\nu}))$ have $u_\phi(\vq, \nu) \leq \bc \cdot (d+1)$, then $\vp \notin \conv(\calP(\bc{,\nu}))$.
\end{proof}

\begin{proof}[Proof of Lemma~\ref{lem:landmark}]
We first show that $\|\Pi_{L_\phi} \vx_t\| \geq \frac{\eps - 2\sqrt{d} \cdot \delta}{4}$. Since for the contexts $\{\vx_t\}_{t \in [\tau]}$ that we consider in epoch $\phi$ it holds that: $w(\Cyl(\calK_\phi,S_\phi),\vx_t) \geq \eps$, then there exists a point $\vp \in \Cyl(\calK_\phi,S_\phi)$ such that $|\langle \vx_t, \vp - \vkappa_\phi \rangle| \geq \frac{\eps}{2}$. Applying the triangle inequality: 
\begin{equation}\label{eq:triangle}
\left|\left \langle \Pi_{L_\phi} \vx_t, \Pi_{L_\phi}(\vp - \vkappa_\phi) \right \rangle \right| + \left| \left\langle \Pi_{S_\phi} \vx_t, \Pi_{S_\phi} (\vp - \vkappa_\phi) \right \rangle \right| \geq \left|\langle \vx_t, \vp - \vkappa_\phi \rangle \right| \geq \frac{\eps}{2}
\end{equation}
Along the directions in $S_\phi$ the following is true: 
$$|\langle \Pi_{S_\phi} \vx_t, \Pi_{S_\phi} (\vp-\vkappa_\phi)\rangle| \leq \|\Pi_{S_\phi} \vx_t\|_2 \cdot \|\Pi_{S_\phi} (\vp-\vkappa_\phi)\|_2 \leq \|\vx_t\|_2 \cdot \sqrt{d} \cdot \|\Pi_{S_\phi} (\vp-\kappa_\phi)\|_\infty \leq 1 \cdot \delta \sqrt{d}$$
Using the latter, Equation~\eqref{eq:triangle} now becomes: 
\begin{equation}\label{eq:triangle2}
\left|\left \langle \Pi_{L_\phi} \vx_t, \Pi_{L_\phi}(\vp - \vkappa_\phi) \right \rangle \right| \geq \frac{\eps}{2} - \sqrt{d} \cdot \delta
\end{equation}
We next focus on upper bounding term $\left|\left \langle \Pi_{L_\phi} \vx_t, \Pi_{L_\phi}(\vp - \vkappa_\phi) \right \rangle \right|$. By applying the Cauchy-Schwarz inequality, Equation~\eqref{eq:triangle2} becomes: 
\begin{equation}\label{eq:step1}
    \left\|\Pi_{L_\phi} \vx_t\right\|_2 \left\|\Pi_{L_\phi} (\vp - \vkappa_\phi)\right\|_2 \geq \left|\left \langle \Pi_{L_\phi} \vx_t, \Pi_{L_\phi} (\vp - \vkappa_\phi) \right \rangle\right| \geq \frac{\eps}{2} - \sqrt{d} \cdot \delta
\end{equation}
For $\|\Pi_{L_\phi}(\vp - \vkappa_\phi)\|_2$, observe that $\vp$ and $\vkappa_\phi$ are inside $\Cyl(\calK_\phi,S_\phi)$, and $\calK_\phi$ has \emph{radius} at most $1$. By the fact that $\vp \in \Cyl(\calK_\phi,S_\phi)$ and Definition~\ref{def:cylindrification}, we can write it as $\vp = \vz + \sum_{i = 1}^{|S_\phi|} y_i \vs_i$ where $\vs_i$ form a basis for $S_\phi$ (which, recall, is orthogonal to $L_\phi$) and $\vz \in \Pi_{L_\phi} \calK_\phi$. Since $\calK_\phi$ is contained in the unit $\ell_2$ ball, we also have that $\Pi_{L_\phi} \calK_\phi$ is contained in the unit $\ell_2$ ball. Hence $\|\Pi_{L_\phi}\vp\|_2 = \|\vz\|_2 \leq 1$. The same holds for $\vkappa_\phi$, and so by the triangle inequality, we have $\|\Pi_{L_\phi}(\vp - \vkappa_\phi)\|_2 \leq 2$. Hence, from Equation~\eqref{eq:step1} we get that: $\|\Pi_{L_\phi} \vx_t\| \geq \frac{\eps - 2\sqrt{d} \cdot \delta}{4}$.

Assume now for contradiction that there does not exist $i$ such that $| \langle \mathbf{e}_i, \vx_t \rangle | \geq \frac{\eps - 2\sqrt{d} \cdot \delta}{4\sqrt{d}}$. This means that for all $j \in [d - |S_\phi|]$ and all contexts $\left\{\vx_t\right\}_{t \in [\tau]}$: $\langle \mathbf{e}_i, \vx_t \rangle < \frac{\eps - 2 \sqrt{d} \cdot \delta}{4\sqrt{d}}$. Denoting by $(E \vx_t)_j$ the $j$-th coordinate of $E \vx_t$ we have that $(E \vx_t)_j = \langle \mathbf{e}_j, \vx_t \rangle$. Hence, if $|\langle \vx_t, \mathbf{e}_j \rangle |< \frac{\eps - 2 \sqrt{d} \cdot \delta}{4\sqrt{d}}$ then: 
$$\|E \vx_t\|_2  = \|\Pi_{L_\phi} \vx_t\|_2 \leq \sqrt{\sum_{i=1}^d (\langle \vx_t, \mathbf{e}_i \rangle)^2} < \sqrt{d \left( \frac{\eps - 2 \sqrt{d} \cdot \delta}{4\sqrt{d}}\right)^2}< \frac{\eps - 2 \sqrt{d} \cdot \delta}{4}$$ which contradicts the fact that $\|\Pi_{L_\phi} \vx_t\| \geq \frac{\eps - 2 \sqrt{d} \cdot \delta}{4}$ established above.
\end{proof}

\subsection{Auxiliary lemmas for Proposition~\ref{prop:known-c} 
}\label{app:computation}

\begin{lemma}[Cap Volume]\label{lem:cap-vol}
With probability at least $\frac{1}{20\sqrt{d-1}}$, a point randomly sampled from a ball of radius $\zeta$ around $\vpst$, $\calB_{L_\phi}(\vpst, \zeta)$, lies on the following halfspace: $\vH^+\left(\vhst_\phi, \left \langle \vhst_\phi, \vpst \right\rangle + \frac{\zeta \cdot \ln(3/2)}{\sqrt{d-1}} \right)$.
\end{lemma}

\begin{proof}
We want to compute the probability that a point randomly sampled from $\calB_{L_\phi}(\vpst, \zeta)$ falls in the following halfspace: \[\vH^+ \equiv \left\{\vx: \left \langle \vhst, \vx - \vpst \right \rangle \geq \frac{\zeta \cdot \ln(3/2)}{\sqrt{d-1}} \right\}\]
Hence, we want to bound the following probability: $\Pr \left[\vx \in \vH^+ | \vx \in \calB(\vpst,\zeta) \right]$.
If we normalize $\calB_{L_\phi}(\vpst,\zeta)$ to be the unit ball $B$, then this probability is equal to: 
\begin{equation}\label{eq:after-rescale}
\Pr \left[ \vx \in \vH^+ | \vx \in \calB_{L_\phi}(\vpst,\zeta) \right] = \Pr \left[ \vx \in \vH^1 | \vx \in B \right] = \frac{\vol\left(B \bigcap \vH^1 \right)}{\vol (B)}
\end{equation}
where $\vH^1$ is the halfspace such that $\vH^1 \equiv \left\{\vx: \langle \vhst, \vx \rangle \geq \frac{\ln(3/2)}{\sqrt{d-1}} = r \right\}$, and the last equality is due to the fact that we are sampling uniformly at random.

Similar to the steps in~\cite[Section~2.4.2]{BHK16}, in order to compute $\vol\left(B \bigcap \vH^1 \right)$ we integrate the incremental volume of a disk with width $dx_1$, with its face being a $(d-1)$-dimensional ball of radius $\sqrt{1-x_1^2}$. Let $V(d-1)$ denote the volume of the $(d-1)$-dimensional unit ball. Then, the surface area of the aforementioned disk is: $(1 - x_1^2)^{\frac{d-1}{2}}\cdot V(d-1)$.
\begin{align*}
\vol\left( B \bigcap \vH^1\right) &= \int_{r}^1 (1 - x_1^2)^{\frac{d-1}{2}} \cdot V(d-1) dx_1 = V(d-1) \cdot \int_{r}^1 (1 - x_1^2)^{\frac{d-1}{2}}dx_1 &\tag{$V(d-1)$ is a constant} \\
        &\geq V(d-1) \cdot \int_{r}^{\sqrt{\frac{\ln 2}{d-1}}} (1 - x_1^2)^{\frac{d-1}{2}}dx_1 &\tag{$\sqrt{\frac{\ln 2}{d-1}} < 1, \forall d \geq 2$}\\
        &\geq V(d-1) \cdot \int_{r}^{\sqrt{\frac{\ln 2}{d-1}}} \left(e^{-2x_1^2}\right)^{\frac{d-1}{2}}dx_1 &\tag{$1 - x^2 \geq e^{-2x^2}, x \in [0, 0.8], \frac{\ln 2}{d-1} \leq 0.8, \forall d \geq 2$} \\
        &= V(d-1) \cdot \int_{r}^{\sqrt{\frac{\ln 2}{d-1}}} e^{-x_1^2(d-1)}dx_1 \\
        &\geq V(d-1) \cdot \int_{r}^{\sqrt{\frac{\ln 2}{d-1}}} \sqrt{\frac{d-1}{\ln 2}} \cdot x_1 \cdot e^{-x_1^2(d-1)}dx_1 &\tag{$x_1 \leq \sqrt{\frac{\ln 2}{d - 1}}$}\\
        &\geq -\frac{V(d-1)}{2\sqrt{(d-1)\cdot \ln 2}} \left[e^{-(d-1)x^2} \right]_{r}^{\sqrt{\frac{\ln 2}{d - 1}}}\\
        &= \frac{V(d-1)}{2\sqrt{(d-1)\cdot \ln 2}} \left(e^{-\ln (\nicefrac{3}{2})} - e^{-\ln 2}\right) = \frac{V(d-1)}{2\sqrt{(d-1)\cdot \ln 2}} \left(\frac{2}{3} - \frac{1}{2} \right) \\ &= \frac{V(d-1)}{12\sqrt{(d-1)\cdot \ln 2}} \numberthis{\label{eq:vol-A}} 
\end{align*}
Next we show how to upper bound the volume of the unit ball $B$. First we compute the volume of one of the ball's hemispheres, denoted be $\vol(H)$. Then, the volume of the ball is $\vol(B) = 2 \vol(H)$. The volume of a hemisphere is \emph{at most} the volume of a cylinder of height $1$ and radius $1$, i.e., $V(d-1) \cdot 1$. Hence, $\vol(B) \leq 2 V(d-1)$. Combining this with Equation~\eqref{eq:vol-A}, Equation~\eqref{eq:after-rescale} gives the following ratio: $$\frac{\vol\left(B \bigcap \vH^1\right)}{\vol (B)} \geq \frac{1}{24\sqrt{(d-1)\cdot \ln 2}} \geq \frac{1}{20\sqrt{d-1}}.$$
This concludes our proof.\end{proof}

\noindent This lower bound on the probability that a randomly sampled point has the large enough margin that Perceptron requires for efficient convergence, suffices for us to guarantee that after a polynomial number of rounds, such a $\tvq$ has been identified in expectation.

\begin{lemma}\label{lem:num-samples}
In expectation, after $N = 20 \sqrt{d-1}$ samples from $\calB_{L_\phi}(\vpst,\zeta)$, at least one of the samples lies in halfspace $\vH^+\left(\vhst_\phi, \left \langle \vhst_\phi, \vpst \right\rangle + \frac{\zeta \cdot \ln(3/2)}{\sqrt{d-1}} \right)$.
\end{lemma}
\begin{proof}
From Lemma~\ref{lem:cap-vol}, the probability that a point randomly sampled from $\calB_{L_\phi}(\vpst,\zeta)$ lies on halfspace $\vH^+\left(\vhst_\phi, \left \langle \vhst_\phi, \vpst \right\rangle + \frac{\zeta \cdot \ln(3/2)}{\sqrt{d-1}} \right)$ is at least $\frac{1}{20\sqrt{d-1}}$. Hence, in expectation after $20\sqrt{d-1}$ samples we have identified one such point by union bound.
\end{proof}

\xhdr{Auxiliary on volumetric progress.} The next lemma states that a convex body $\calK$ with width at least $\delta$ in every direction must fit a ball of diameter $\delta/d$ inside it.

\begin{lemma}[{{\cite[Lemma~6.3]{LLV18}}}]\label{lem:ball-inside}
If $\calK \subset \bbR^d$ is a convex body such that $w(\calK, \vu) \geq \delta$ for every unit vector $\vu$, then $\calK$ contains a ball of diameter $\delta/d$.
\end{lemma}

\begin{lemma}[Directional Gr{\"u}nbaum {\cite[Theorem~5.3]{LLV18}}]\label{lem:dir-grun}
If $\calK$ is a convex body and $\vkappa$ is its centroid, then, for \emph{every} unit vector $\vu \neq 0$, the set $\calK_+ = \calK \bigcap \{\vx | \langle \vu, \vx -\vkappa \rangle \geq 0\}$ satisfies: \[\frac{1}{d+1} w(\calK,\vv) \leq w(\calK_+, \vv) \leq w(\calK,\vv), \quad \text{for all unit vectors } \vv. \]
\end{lemma}

The Approximate Gr{\"u}nbaum lemma, which is stated next, relates the volume of a set $\calK_+^{\mu} = \{\vx \in \calK: \langle \vu, \vx - \vkappa \rangle \geq \mu\}$ with the volume of set $\calK$, when $\mu \leq 1/d$ for any unit vector $\vu$. Its proof, which we provide right below, is similar to the proof of \cite[Lemma~5.5]{LLV18} with the important difference that $\mu$ is no longer $w(\calK,\vu)/(d+1)^2$, but rather, $\mu < 1/d$

\begin{lemma}[Approximate Gr{\"u}nbaum]\label{lem:approx_grunbaum}
Let $\calK$ be a convex body and $\vkappa$ be its centroid. For an arbitrary unit vector $\vu$ and a scalar $\mu$ such that $0 < \mu < \frac{1}{d}$, let $\calK_+^{\mu} = \{\vx \in \calK: \langle \vu, \vx - \vkappa \rangle \geq \mu\}$. Then: $\vol\left(\calK_+^{\mu}\right) \geq \frac{1}{2e^2}\vol(\calK)$.
\end{lemma}

In order to prove the Appoximate Gr{\"u}nbaum lemma we make use of Brunn's theorem and the Gr{\"u}nbaum Theorem, both stated below.

\begin{lemma}[Brunn's Theorem]\label{lem:brunn}
For convex set $\calK$ if $g(x)$ is the $(d-1)$-dimensional volume of the section $\calK \bigcap \{\vy | \langle \vy, \mathbf{e}_i \rangle = x\}$, then the function $r(x) = g(x)^{\frac{1}{d-1}}$ is concave in $x$ over its support.
\end{lemma}

\begin{lemma}[Gr{\"u}nbaum Theorem]\label{thm:std-grun}
Let $\calK$ denote a convex body and $\vkappa$ its centroid. Given an arbitrary non-zero vector $\vu$, let $\calK_+ = \{\vx| \langle \vu, \vx - \vkappa \rangle \geq 0\}$. Then: 
$$\frac{1}{e}\vol(\calK) \leq \vol \left(\calK_+ \right) \leq \left(1 - \frac{1}{e} \right) \vol (\calK)$$
\end{lemma}

\begin{proof}[Proof of Lemma~\ref{lem:approx_grunbaum}]
For this proof we assume without loss of generality that $\vu =e_1$, and that the projection of $\calK$ onto $e_1$ is interval $[a,1]$. We are interested in comparing the following two quantities: $\vol(\calK)$ and $\vol(\calK_+^\mu)$. By definition: 
\begin{equation}
\vol \left( \calK_+ \right) = \int_0^1 r(x)^{d-1}dx \qquad \text{and} \qquad \vol \left( \calK_+^{\mu} \right) = \int_\mu^1 r(x)^{d-1}dx
\end{equation}
where $r(x) = g(x)^{\frac{1}{d-1}}$ and $g(x)$ corresponds to the volume of the $(d-1)$-dimensional section $\calK_x = \calK \bigcap \{\vx| \langle \vx, e_i \rangle = x \}$. We now prove that $\vol(\calK_+^\mu) \geq \frac{1}{e} \vol(\calK_+)$. Combining this with Gr\"unbaum Theorem (Lemma~\ref{thm:std-grun}) gives the result. We denote by $\rho$ the following ratio: 
\begin{equation}\label{eq:rho-1}
    \rho = \frac{\int_\mu^1 r(x)^{d-1}dx}{\int_0^1 r(x)^{d-1}dx} \geq \frac{\int_{\nicefrac{1}{d}}^1 r(x)^{d-1}dx}{\int_0^1 r(x)^{d-1}dx}
\end{equation}
We approximate function $r(x)$ with function $\tr$: $$\tr(x) = \twopartdef{r(x)}{0 \leq x \leq \delta}{(1-x) \cdot \frac{r(\delta)}{1-\delta}}{\delta <x \leq 1}$$
Note that since $0= \tr(1) \leq r(1)$ (because $r(x)$ is a non-negative function) and $r(x)$ is concave from Brunn's theorem (Lemma~\ref{lem:brunn}), for functions $r(x)$ and $\tr(x)$ it holds that $r(x) \geq \tr(x)$. Using this approximation function $\tr(x)$ along with the fact that function $f(z) = \frac{z}{y+z}$ is \emph{increasing} for any scalar $y >0$, we can relax Equation~\eqref{eq:rho-1} as follows: 
\begin{equation}\label{eq:rho-2}
    \rho \geq \frac{\int_{\nicefrac{1}{d}}^1 \tr(x)^{d-1}dx}{\int_0^{\nicefrac{1}{d}} \tr(x)^{d-1}dx + \int_{\nicefrac{1}{d}}^1 \tr(x)^{d-1}dx}
\end{equation}
Next, we use another approximation function $\hr(x) = (1-x)\cdot \frac{r(\delta)}{1-\delta}, 0 \leq x \leq 1$; this time in order to approximate function $\tr(x)$. For $x \in [\delta, 1]$: $\tr(x) = \hr(x)$. For $x \in [0,\delta]$ and since $\tr(0) = r(0) = 0$ and $\tr(x)$ is concave in $x \in [0,\delta]$, $\hr(x) \geq \tr(x) = r(x), x \in [0,\delta]$. Hence, Equation~\eqref{eq:rho-2} can be relaxed to: 
\begin{align*}
    \rho &\geq \frac{\int_{\nicefrac{1}{d}}^1 \hr(x)^{d-1}dx}{\int_0^{\nicefrac{1}{d}} \hr(x)^{d-1}dx + \int_{\nicefrac{1}{d}}^1 \hr(x)^{d-1}dx} =  \frac{\int_{\nicefrac{1}{d}}^1 (1-x)^{d-1}\cdot \left( \frac{r(\delta)}{1-\delta}\right)^{d-1}dx}{\int_{0}^1 (1-x)^{d-1}\cdot \left( \frac{r(\delta)}{1-\delta}\right)^{d-1}dx}\\
    &= \frac{\int_{\nicefrac{1}{d}}^1  (1-x)^{d-1} dx}{\int_{0}^1 (1-x)^{d-1} dx} = \frac{-\frac{1}{d} \left(0 - \left(1 - \frac{1}{d}\right)^d \right)}{-\frac{1}{d} \left(0 - 1\right)} = \left( 1 - \frac{1}{d}\right)^d \geq \frac{1}{2e} 
\end{align*}
This concludes our proof.\end{proof}

We next state the cylindrification lemma, whose proof was provided by \cite{LLV18}, relates the volume of the convex body to the volume of its projection onto a subspace. 

\begin{lemma}[{Cylindrification \cite[Lemma~6.1]{LLV18}}]\label{lem:cyl}
Let $\calK$ be a convex body in $\bbR^d$ such that $w(\calK,\vu) \geq \delta'$ for every unit vector $\vu$. Then, for every $(d-1)$-dimensional subspace $L$ it holds that $\vol(\Pi_L \calK) \leq \frac{d(d+1)}{\delta'} \vol(\calK)$.
\end{lemma}

\begin{lemma}[Epoch Based Projected Gr{\"u}nbaum]\label{lem:epoch-grun}
For $\delta = \frac{\eps}{4(d+\sqrt{d})}$ and $\calK_{\phi+1} = \calK_\phi \bigcap \vH^+(\tvh_\phi, \tomega_\phi)$, where $\vH^+(\tvh_\phi,\tomega_\phi)$ was the halfspace returned from \textsc{CorPV.SeparatingCut} it holds that: 
\[\vol\left(\Pi_{L_{\phi}}\calK_{\phi+1}\right)\leq \left(1-\frac{1}{2e^2}\right)\vol(\Pi_{L_{\phi}}\calK_{\phi})\]
\end{lemma}

\begin{proof}
By Lemma~\ref{lem:computation}, we know that \textsc{CorPV.SeparatingCut} returned hyperplane $(\tvh_{\phi},\tomega_\phi)$ orthogonal to all small dimensions, such that $\dist(\vkappast_\phi, (\tvh_\phi, \tomega_\phi)) \leq 3 \bnu = 3\cdot \frac{\eps - 2\sqrt{d}\delta}{4\sqrt{d}}$. Substituting $\delta = \frac{\eps}{4(d + \sqrt{d})}$ we get that: \[\dist\left(\vkappast_\phi, \left(\tvh_\phi,\tomega_\phi\right)\right) \leq \frac{(2\sqrt{d} +1) \eps}{2\sqrt{d}(\sqrt{d}+1)} \leq \frac{1}{d}\] 
where the last inequality uses the fact that $\eps \leq \nicefrac{1}{\sqrt{d}}$ and that $\frac{2\sqrt{d}+1}{\sqrt{d}+1}\leq 2$. Hence, the clause in the approximate Gr\"unbaum lemma (Lemma~\ref{lem:approx_grunbaum}) holds and as a result, applying the approximate Gr\"unbaum lemma with $\calK=\Pi_{L_\phi}\calK_\phi$, the lemma follows. \end{proof}

For completeness, we state the celebrated Perceptron mistake bound lemma \citep{N63}.

\begin{lemma}\label{lem:perceptron}
Given a dataset $\calD = \{ \vx_i, y_i \}_{i \in [n]}$ with $\vx_i \in \bbR^d$ and $y_i \in \{-1, +1\}$, if $\|x_i\| \leq R$ and there exists a linear classifier $\btheta$ such that $\|\btheta \| = 1$ and $y_i \cdot \langle \btheta, \vx_i \rangle \geq \gamma$ for a scalar $\gamma$. Then, the number of mistakes that the Perceptron algorithm incurs in $\calD$ is upper bounded by $(R/\gamma)^2$.
\end{lemma}

%% file: app_bounded_rationality.tex
\section{Extension to bounded rationality.} \label{sec:extensions_behavior}
We now extend the algorithm and analysis to the bounded rationality behavioral model. We first recap the behavioral model. There is a noise parameter $\xi_t$ drawn from a $\sigma$-subgaussian distribution $\subG(\sigma)$, \emph{fixed} across rounds and \emph{known} to the learner, i.e., nature selects it before the first round and reveals it. At every round $t$ a realized noise $\xi_t\sim \subG(\sigma)$ is drawn, but $\xi_t$ is never revealed to the learner. The agent's perceived value is then $\tv_t=v(\vx_t)+\xi_t$.

We focus on a \emph{pseudo-regret} definition that compares to a benchmark that has access to $\stheta$ and $\subG(\sigma)$ but does not have access to the realization $\xi_t$. The resulting benchmark is:
\begin{equation}\label{eq:bench}
   L_{\stheta}^{\star}(\vx) = \min_{\omega^{\star}} \E_{\xi'\in\subG(\sigma)}\big[\ell(\omega^{\star},\left\langle \vx,\stheta \right\rangle, \langle \vx, \stheta\rangle + \xi')\big].
\end{equation}
and the corresponding regret is $R(T)=\sum_{t\in[T]}\big[\ell(\omega_t,v(\vx_t),\tv_t)-L_{\stheta}^{\star}(\vx_t)\big]$. 

We remark that $\omegast$ should be thought of as the optimal query that the learner could have issued had we known $\stheta$ but not the realization of $\xi'$. To develop more intuition regarding the benchmark stated assume for example that $\xi'$ comes from a normal distribution. Then, the optimal $\omegast$ in expectation for the $\eps$-ball and the absolute loss is equal to $\langle \vx, \stheta \rangle$. However, $\omegast$ should be strictly \emph{lower} than $\langle \vx, \stheta \rangle$ when interested in the pricing loss, due to its discontinuity.

Our algorithm only differs from the one described in Section~\ref{sec:corpv_algorithm} in the $\textsc{Exploit}$ module (Algorithm~\ref{alg:corpv-exploit}) as $\omega_t$ is defined in a similar way with the benchmark. More formally, we again consider the worst-case selection of $\btheta$ consistent with the knowledge set and select the query that minimizes our loss with respect to that, i.e., $\omega_t=\argmin_{\omega}\max_{\btheta\in\calK_\phi} \E_{\xi'\in\subG(\sigma)}\big[\ell(\omega,\langle \vx_t, \stheta\rangle, \langle \vx_t,\stheta,\vx_t\rangle +\xi')\big]$. The algorithm also doubles the corruption budget it should be robust to ($\bc$ in Algorithm~\ref{algo:corpvAI}) and takes care of the additional noise by treating its tail as corruption and upper bounding it by $\bc$. 

\begin{theorem}\label{thm:bounded} With probability at least $1-2\beta$, the guarantee of Theorem~\ref{thm:agnostic-corpv} extends to when rounds with fully rational agents are replaced by boundedly rational with $\sigma \leq \frac{\eps}{8 \sqrt{2d} (\sqrt{d} + 1) \ln T}$. 
\end{theorem}
We note that Corollary 1 in \citep{CLL16} has a regret of $\bigO(d^2\log T)$ for pricing loss with $\sigma \approx \frac{d}{T\log T}$. For pricing loss, $\eps=\frac{1}{T}$ and our bound is weaker by a factor of $d$ on the regret and a factor $d^2$ on the subgaussian variance $\sigma$, but it allows for the simultaneous presence of adversarially corrupted
agents.

\label{app:bounded}

\begin{proof}[Proof of Theorem~\ref{thm:bounded}]
We first show that under the low-noise regime stated above, the noise is bounded by $\Xi = \sqrt{2} \sigma \ln T$ with high probability at every round. Indeed, by Hoeffding's inequality we have that $\Pr\left[ |\xi_t | > \Xi \right] \leq e^{-\ln^2 T}$. Using the union bound we have: $\Pr \left[ |\xi_t| > \Xi, \text{ for any } t \in [T] \right] \leq \beta'=\beta/T$, and so $\Pr \left[ |\xi_t| \leq \Xi, \forall t \in [T] \right] \geq 1 - \beta$, which contributes the additional $\beta$ in the high-probability argument.

We next show that when $\sigma \leq \frac{\eps}{8\sqrt{2d}(\sqrt{d}+1)\ln T}$, then our algorithm maintains $\stheta$ in $\calK_\phi$. This is enough to ensure that the regret guarantee remains order unchanged.
Since the perceived value of \br agents is $\tv_t = v(\vx_t) + \xi_t$, then, in order to ``protect'' $\stheta$ (i.e., make sure that $\stheta \in \calK_{\phi+1}$) we need the hyperplanes that we feed to \textsc{CorPV.SeparatingCut} to have a margin of $\Xi$ (since $\xi_t \leq \Xi$). To do so, it suffices to slightly change the lower bound of $\nu$ for the $\nu$-margin projected undesirability levels that we use throughout the proof such that the new lower bound is $\underline{\nu}' = \underline{\nu} + \Xi = \sqrt{d} \cdot \delta + \Xi$. Since $\nu$ is such that $\underline{\nu}' \leq \nu \leq \bnu$, then it must the case that $\underline{\nu}' = \sqrt{d} \delta + \Xi \leq \bnu = \frac{\eps (2 \sqrt{d} + 1)}{8\sqrt{d} (\sqrt{d} + 1)}$. Solving for $\Xi$ we obtain the result. This concludes our proof. 
\end{proof}

%% file: app_discussion.tex
\section{Discussion of algorithmic choices affecting the regret guarantee}\label{app:discussion}\label{ssec:improper_cut}

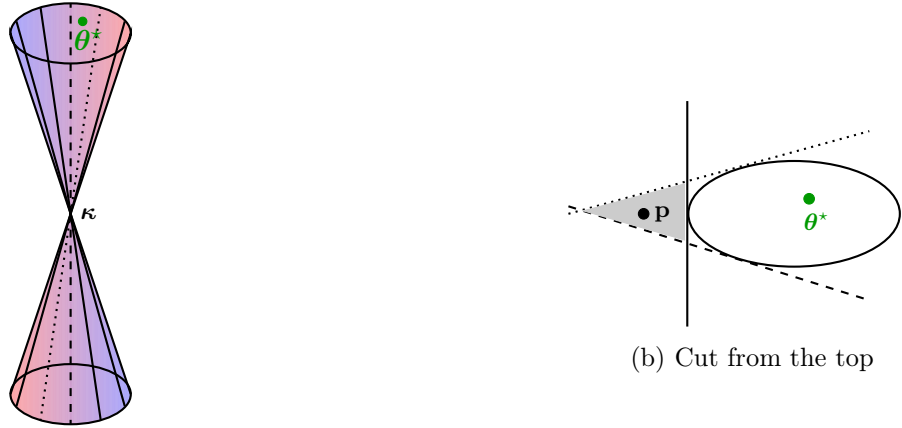
\begin{figure}[htbp]
\begin{subfigure}{0.6\textwidth}
\centering
    \begin{tikzpicture}[scale=0.8] 
        \shade[left color=blue!40!white,opacity=0.75, right color = red!40!white,opacity=0.75] (-1,0) arc (180:0:1cm and 0.5cm) -- (0,-3) -- cycle;
        \draw [thick] (-1,0) arc (180:360:1cm and 0.5cm) -- (0,-3) -- cycle;
        \draw [thick] (-1,0) arc (180:0:1cm and 0.5cm);
        \draw [thick, dashed] (0, -3) -- (0,   0.5); 
        \draw [thick, dotted] (0, -3) -- (0.5, 0.43);
        \draw [thick] (0, -3) -- (-0.5, 0.43);
        \draw [thick] (0, -3) -- (0.9, 0.2);
        \draw [thick] (0, -3) -- (-0.9, 0.2);
        \begin{footnotesize}
        \draw[color=black] (0.3, -3)    node {$\vkappa$}; 
        \end{footnotesize}
        \shade[left color=red!40!white,opacity=0.75, right color=blue!40!white,opacity=0.75] (-1,-6) arc (180:360:1cm and 0.5cm) -- (0,-3) -- cycle;
        \draw [thick] (-1,-6) arc (180:360:1cm and 0.5cm) -- (0,-3) -- cycle;
        \draw [thick] (-1,-6) arc (180:0:1cm and 0.5cm);
        \draw [thick, dashed] (0, -3) -- (0,   -6.5); 
        \draw [thick] (0, -3) -- (0.5, -6.43);
        \draw [thick,dotted] (0, -3) -- (-0.5, -6.43);
        \draw [thick] (0, -3) -- (0.9, -6.2);
        \draw [thick] (0, -3) -- (-0.9, -6.2);
        \draw [color = green!60!black, fill = green!60!black] (0.2,0.2) circle (2pt);
        \draw[color = green!60!black] (0.3, -0.1) node {$\stheta$};
    \end{tikzpicture}
    \caption{3D view of the hyperplane ``cuts'' created by contexts}\label{fig:three_dim_a}
    \end{subfigure}\hfill
    \begin{subfigure}{0.3\textwidth}
    \begin{tikzpicture}
    \draw [thick] (0,0) ellipse (40pt and 20pt);  
    \draw [color = green!60!black, fill = green!60!black] (0.2, 0.2) circle (2pt);
    \draw [thick] (-1.42, -1.5) -- (-1.42, 1.5);
    \draw [thick, dotted] (-3, 0) -- (1, 1.1);
    \draw [thick, dashed] (-3, 0.1) -- (1, -1.15);
    \draw [color=black, fill = black] (-2,0) circle (2pt);
    \draw [color = black!20!white, fill = black!20!white] (-1.45,-0.35) -- (-2.8,0.04) -- (-1.45, 0.4);
        \draw [color=black, fill = black] (-2,0) circle (2pt);
    \begin{footnotesize}
        \draw [color=green!60!black] (0.3, -0.1)  node {$\stheta$}; 
        \draw [color=black, fill=black] (-1.75, 0) node {$\mathbf{p}$};
        \end{footnotesize}
    \end{tikzpicture}
    \caption{\cpedit{Cut} from the top}\label{fig:three_dim_b}
    \end{subfigure}
    \caption{Sketch on why proper cuts do not suffice.}\label{fig:three_dim}
\end{figure}

Before concluding this section, we discuss algorithmic choices affecting the regret guarantee. 

\xhdr{Not making updates by single context.} All the contextual search approaches that are based on binary search techniques rely on refining a similarly constructed knowledge set (or, in other words, version space) that contains the ground truth $\stheta$. In all of the previous non-corrupted works, one could just use every explore query in order to refine this knowledge set. However, in our corrupted setting, this technique can result in the removal of $\stheta$ from the knowledge set (as can be be seen by a simple one-dimensional example formalized in Appendix~\ref{app:single_dimension}). As a result, to employ such binary-search techniques, we need to be more careful about when and how we refine the knowledge set. Our approach is to only remove from the knowledge set parameters that are certifiably not $\stheta$. To simplify the subsequent discussion, assume that contexts $\{\vx_t\}_{t \in [\tau]}$ lead only to \emph{explore} queries and that we are still interested in the simpler $\bc$-known corruption setting with $\bc = 1$.

\xhdr{Creating a separating cut by combining explore queries.} Ideally, if we could identify one context $\vx \in \{\vx_t\}_{t \in [\tau]}$ such that the $\bc$-protected region $\calP(\bc,\nu)$ is inside the halfspace $\vH^+\left(\vx, \vkappa_\phi \right)$, i.e., $\calP(\bc, \nu) \subseteq \vH^+(\vx,\vkappa_\phi)$, then we could update the knowledge set as $\calK_\phi \bigcap \vH^+(\vx,\vkappa_\phi)$. As we have explained, these properties ensure sufficient volumetric progress.
In $d = 2$, indeed one of the contexts among $\{\vx_t\}_{t \in [\tau]}$ has the aforementioned property due to a monotonicity argument that we describe in Appendix~\ref{app:two-dim}. 

However, this is no longer true in $d=3$, even if one sees arbitrarily many contexts in an epoch. To see this, consider Figure~\ref{fig:three_dim} and assume that all rounds are \emph{uncorrupted}. In Figure~\ref{fig:three_dim_a}, each straight line corresponds to a context $\vx_t$ and the shaded region corresponds to the halfspace with feedback $y_t = +1$, forming this ``cone.'' In Figure~\ref{fig:three_dim_b} we visualize a cross section of the knowledge set shown in Figure~\ref{fig:three_dim_a} and zoom in on only $3$ of the halfspaces around $\stheta$; the dotted, the dashed and the solid. 

We are going to reason about the undesirability level of points like $\vp$, lying in the shaded area of Figure~\ref{fig:three_dim_b}. Points $\vp$ and $\stheta$ lie on the same side of both the dashed and dotted hyperplanes, and so these two do not contribute any undesirability to $\vp$. The solid line does contribute once to the undesirability level of $\vp$ (and all the points in the shaded region). Recall that since $\bc = 1$, we need a hyperplane with undesirability at least $2$ in the \emph{entirety} of one of its halfspaces. However, for any number of contexts, we can form the cone structure in Figure~\ref{fig:three_dim_a}, in which, for every hyperplane, there exists a shaded region like the one in Figure~\ref{fig:three_dim_b}, whose points have undesirability $1$. 

On the other hand, there exists another hyperplane (not associated with a single explore query) with undesirability at least $\bc+1$. This is the cut that separates the upper part of the cone in Figure~\ref{fig:three_dim_a} containing $\stheta$ from the lower part. 

\xhdr{Separating into small and large dimensions.} In the beginning of this discussion, we assumed that we will focus only on \emph{explore} queries. This is important, as cuts that are made in directions of small width do not necessarily adequately refine our estimate for $\stheta$. That said, since we have established that the cut we make may not correspond to any of the observed contexts, we cannot automatically guarantee that the width of the direction for that cut will indeed be large enough.

To deal with this, we separate the dimensions into small and large and project all objects into the subspace spanned by the ``large dimensions''. This guarantees that any cut that we create will also live in the large dimension subspace and will therefore have sufficiently large width to enable adequate volumetric progress. This is the place where our approach is tailored to the \textsc{ProjectedVolume} algorithm rather than being a generic reduction to any binary-search method for the uncorrupted case as, to the best of our knowledge, the \textsc{ProjectedVolume} is the only one that explicitly separates small and large dimensions, which makes it amenable for our purposes.

\xhdr{Employing landmarks.} So far we have clarified the need to handle small and large dimensions of the knowledge set separately. As a next step, Carath{\'e}odory's Theorem provides an upper bound in the undesirability of all the points within $\conv(\calP(\bc))$. The last step is to identify at least one point \emph{in the large dimensions} that has undesirability strictly larger than the bound provided by Carath{\'e}odory's theorem. This is where our ``landmarks'' construction comes into play and serves the following dual purpose. On the one hand, at least one of them has large enough undesirability that it cannot be in $\conv(\calP(\bc))$ and hence, this landmark can be separated from $\conv(\calP(\bc))$ using the Perceptron algorithm. On the other hand, because of their construction, the hyperplane returned by Perceptron is guaranteed to be valid, meaning that the knowledge set has large width in its direction.

\begin{remark}
We note that the additional $d^2$ degradation in the regret compared to the uncorrupted case arises from the use of Carath\'eodory's Theorem and the use of landmarks respectively. We view the former as inherent to our approach and therefore achieving a linear dependence on $d$ would require fundamentally new ideas. Regarding the latter, landmarks may be an artifact of the particular analysis and there may be other ways to identify such a highly undesirable point while still retaining the main principles of our methodological approach.
\end{remark}

\xhdr{Further details on the need to combine multiple explore queries.}\label{app:two-dim}
In order to prove the results of this section, we use a simplified version of undesirability levels; we define a point's $\vp \in \calK_\phi$ undesirability level as the number of rounds within epoch $\phi$, for which \[u_{\phi}(\vp) = \sum_{t \in [\tau]} \1 \left\{ \left\langle \vp - \vkappa_\phi, \vx_t \right\rangle \cdot y_t < 0 \right\}.\]

\noindent We next present two propositions regarding the number of contexts needed in order to guarantee that we have found an appropriately undesirable hyperplane, for the cases of $d=2$ and $d=3$ respectively.
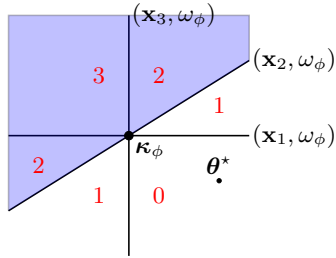
\begin{figure}[htbp]
\centering
\begin{tikzpicture}[line cap=timestep,line join=timestep,>=latex',x = 1cm, y = 1cm, scale=0.4]
\draw[semithick] (-4,0) -- (4,0);
\draw[semithick] (-4,-2.5) -- (4,2.5);
\draw[semithick] (0,-4) -- (0,4);
\draw[semithick, fill = blue, opacity = 0.25]  (-4, -2.5) -- (-4,4) -- (4,4) -- (4,2.5);
\draw[color=black,fill=black] (0,0) circle (4pt);
\draw[color=black,fill=black] (3,-1.5) circle (2pt);
\begin{footnotesize}
\draw[color=black] (0.7,-0.5) node {$\vkappa_\phi$}; 
\draw[color=black] (5.5, 0) node {$(\vx_1, \omega_\phi)$};
\draw[color=black] (5.5, 2.5) node {$(\vx_2, \omega_\phi)$};
\draw[color=black] (1.5, 4) node {$(\vx_3, \omega_\phi)$};
\draw[color=black] (3, -1) node {$\stheta$};
\draw[color=red]   (3, 1)  node {$1$};
\draw[color=red]   (1, 2)  node {$2$};
\draw[color=red]   (-1, 2)  node {$3$};
\draw[color=red]   (-3, -1)  node {$2$};
\draw[color=red]   (-1, -2)  node {$1$};
\draw[color=red]   (1, -2)  node {$0$};
\end{footnotesize}
\end{tikzpicture}
\caption{Sketch of the undesirability levels for epoch $\phi$, after $2\bc + 1$ uncorrupted rounds, assuming that each context appears once. Red numbers denote the undesirability level of each region. The opaque region denotes the knowledge set for epoch $\phi+1$.}\label{fig:2d}
\end{figure}

\begin{proposition}\label{prop:2d}
For $d = 2$ and any corruption level $\bc$, after $3\bc + 1$ rounds within an epoch, there exists a hyperplane $(\vx', \omega')$ among $\left\{(\vx_t, \omega_t)\right\}_{t \in [\tau]}$ with undesirability level at least $\bc +1$ in the entirety of one of its halfspaces.
\end{proposition}

\begin{proof}
Since there exist at most $\bc$ corrupted rounds among the $3\bc + 1$ rounds of epoch $\phi$, then at least $2\bc + 1$ are \emph{uncorrupted}. We say that these rounds are part of the set $U_\phi$. For all $t \in U_\phi$, the learner's hyperplanes $\{(\vx_t, \omega_t)\}_{t \in U_\phi}$ pass from the same centroid $\vkappa_\phi$ and they all \emph{protect} the region where $\stheta$ lies. In other words, none among $\{(\vx_t, \omega_t)\}_{t \in U_\phi}$ adds an undesirability point to $\stheta$ (see e.g.,  Figure~\ref{fig:2d} for $\bc = 1$ and each context appears only once). Since all hyperplanes point towards the same direction (i.e., the region containing $\stheta$ never gets an undesirability point), starting from the region where $\stheta$ lies and moving counter clockwise the undesirability levels of the formed regions first increase (moving from $0$ to $2\bc + 1$) and then decrease (moving from $2\bc + 1$ to $0$). Due to this being a concave function, it is clear to see that there always exists a hyperplane with undesirability level at least $\bc + 1$ in the entirety of one of its halfspaces. \end{proof}

\begin{proposition}\label{prop:no-existing-hyper}
For $d = 3$, any corruption $\bc$, any centroid $\vkappa$, and any number of rounds $N$ within an epoch, there exists a $\stheta$ and a sequence $\{\vx_t\}_{t \in [N]}$, such that there does not exist a hyperplane $(\vx', \omega')$, where $\vx' \in \{\vx_t \}_{t \in [N]}$, with one of its halfspaces having undesirability at least $\bc + 1$. 
\end{proposition}

\begin{proof}
For any convex body $\calK$ with centroid $\vkappa$, we show how to construct a problematic instance of a $\stheta$ and $N$ contexts. Fix the corruption level to be $\bc = 1$, and $c_t = 0, \forall t \in [N]$. However, the learner does not know that none of the rounds is corrupted. Construct a sequence of contexts $\{\vx_t\}_{t \in [N]}$ such that no two are equal and for $\omega_t = \langle \vx_t, \vkappa \rangle, \forall t \in [N]$ we have that: \[\left\{\left(\vx_{t_1}, \omega_{t_1} \right)\right\} \bigcap \left\{\left(\vx_{t_2}, \omega_{t_2} \right)\right\} = \vkappa\] and the smallest region $r^{\star}$ that contains $\stheta$ is defined by all $\{\vx_t\}_{t \in [N]}$. Intuitively, these hyperplanes form a conic hull. 

Take any hyperplane $h \in \bbR^3$ neither parallel nor orthogonal with any hyperplane $\{(\vx_t, \omega_t)\}_{t \in [N]}$ such that $h \cap r^{\star} = q \neq \emptyset$. Take $q$'s projection in $\bbR^2$. Observe that we have constructed an instance where no matter how big $N$ is, there does not exist any hyperplane with undesirability at least $\bc + 1$ (i.e., $2$ when $\bc = 1$) in either one of its halfspaces. This instance easily generalizes for any $\bc > 1$. \end{proof}

%% file: app_GD.tex
\section{Proof from Section~\ref{sec:gradient_descent}}\label{app:GD}

\begin{proof}[Proof of Theorem~\ref{thm:GD-FR}]
Function $f_t(\vz) = - y_t \cdot \langle \vz, \vx_t \rangle$ is Lipschitz in $\vz$. So, using the known guarantees for Online Gradient Descent and denoting by $\vz^{\star} = \arg\min_{\vz} \sum_{t \in [T]} f_t(\vz)$ we know that:
\begin{equation}\label{eq:std-gd}
    \sum_{t \in [T]} f_t(\vz_t) - \sum_{t \in [T]} f_t(\vz^{\star}) = \calO \left( \sqrt{T} \right)
\end{equation}
Due to the definition of $\vz^{\star}$ we can relax the left-hand side of Equation~\eqref{eq:std-gd} and get: 
\begin{equation}\label{eq:std-gd2}
    \sum_{t \in [T]} f_t(\vz_t) - \sum_{t \in [T]} f_t(\stheta) \leq \calO \left( \sqrt{T} \right)
\end{equation}
We now analyze the quantity on the left-hand side of Equation~\eqref{eq:std-gd2} as follows: 
\begin{equation}\label{eq:gd-abs}
    \sum_{t \in [T]} f_t(\vz_t) - \sum_{t \in [T]} f_t(\stheta) = \sum_{t \in [T]} y_t \cdot \left( \langle \stheta, \vx_t \rangle - \langle \vz_t, \vx_t \rangle \right) = \sum_{t \in [T]} |\langle \stheta, \vx_t \rangle - \omega_t|
\end{equation}
which is the quantity that we wish to minimize when we are trying to minimize the absolute loss given binary feedback when the round is not corrupted. Given the fact that $y_t$ is arbitrary for at most $C$ rounds, the regret incurred by $\gd$ is at most $\calO(\sqrt{T} + C)$.

For the proof of the $\eps$-ball loss, we show how the latter compares with the absolute loss. Indeed: 
\begin{align*}
    \sum_{t \in [T]} |\langle \stheta, \vx_t \rangle - \omega_t| \geq \eps \sum_{t \in [T]} \1 \left\{ |\langle \stheta, \vx_t \rangle - \omega_t | \geq \eps \right\}
\end{align*}
Combining the above with Equation~\eqref{eq:gd-abs} we establish that $\gd$ for the $\eps$-ball loss incurs regret $\calO(\sqrt{T}/\eps + C/\eps)$
\end{proof}